%% file: main.tex
\documentclass{article}
\usepackage{ijcai25}

\usepackage{times}
\usepackage{soul}
\usepackage{url}
\usepackage[hidelinks]{hyperref}
\usepackage[utf8]{inputenc}
\usepackage[small]{caption}
\usepackage{graphicx}
\usepackage{amsmath}
\usepackage{amsthm}
\usepackage{booktabs}
\usepackage{algorithm}
\usepackage{algorithmic}
\usepackage[switch]{lineno}

\usepackage{amssymb}
\usepackage{subfig,graphicx}
\usepackage{xcolor}
\usepackage{xspace}
\usepackage{stackengine}
\usepackage{booktabs}
\usepackage{makecell}
\usepackage{soul}
\usepackage{amsthm}
\usepackage{thm-restate}
\usepackage{cprotect}
\usepackage{wrapfig}  \usepackage{float}
\newcommand{\method}{S-EPOA\xspace}

\definecolor{lightred}{RGB}{250, 220, 220}
\definecolor{lightblue}{RGB}{200, 240, 250}

\title{\method: Overcoming the Indistinguishability of Segments with \\
Skill-Driven Preference-Based Reinforcement Learning}

\author{
Ni Mu$^1$\thanks{These authors contributed equally.}\and
Yao Luan$^1$\footnotemark[1]\and
Yiqin Yang$^{2}$\thanks{Corresponding Authors.}\and
Bo Xu$^2$\And
Qing-Shan Jia$^1$\footnotemark[2]
\\
\affiliations
$^1$Department of Automation, Tsinghua University\\
$^2$Institute of Automation, Chinese Academy of Sciences\\
\emails
\{mn23, luany23\}@mails.tsinghua.edu.cn,
yiqin.yang@ia.ac.cn, 
jiaqs@tsinghua.edu.cn
}

\begin{document}

\maketitle

\input{text/0_abstract}

\input{text/1_intro}

\input{text/3_preliminaries}

\input{text/4_gap}

\input{text/5_method}

\input{text/6_experiments}

\input{text/2_related_work}
\input{text/7_conclusion}

\section*{Acknowledgments}
This work is supported by the NSFC (No. 62125304, 62192751, and 62073182), the Beijing Natural Science Foundation (L233005), BNRist project (BNR2024TD03003), and the 111 International Collaboration Project (B25027).

\bibliographystyle{ijcai25_named}
\bibliography{ref}

\clearpage
\appendix
\onecolumn
\input{appendix/1_proof}
\input{appendix/2_experimental_details}
\input{appendix/3_human_exp}

\input{appendix/4_skill_discovery}

\end{document}

%% file: text/0_abstract.tex
\begin{abstract}
Preference-based reinforcement learning (PbRL) stands out by utilizing human preferences as a direct reward signal, eliminating the need for intricate reward engineering.
However, despite its potential, traditional PbRL methods are often constrained by the indistinguishability of segments, which impedes the learning process.
In this paper, we introduce Skill-Enhanced Preference Optimization Algorithm~(\method), which addresses the segment indistinguishability issue by integrating skill mechanisms into the preference learning framework.
Specifically, we first conduct the unsupervised pretraining to learn useful skills.
Then, we propose a novel query selection mechanism to balance the information gain and distinguishability over the learned skill space.
Experimental results on a range of tasks, including robotic manipulation and locomotion, demonstrate that \method significantly outperforms conventional PbRL methods in terms of both robustness and learning efficiency. 
The results highlight the effectiveness of skill-driven learning in overcoming the challenges posed by segment indistinguishability. 
\end{abstract}

%% file: text/1_intro.tex
\begin{figure*}[t]
    \centering
    \includegraphics[width=0.85\linewidth]{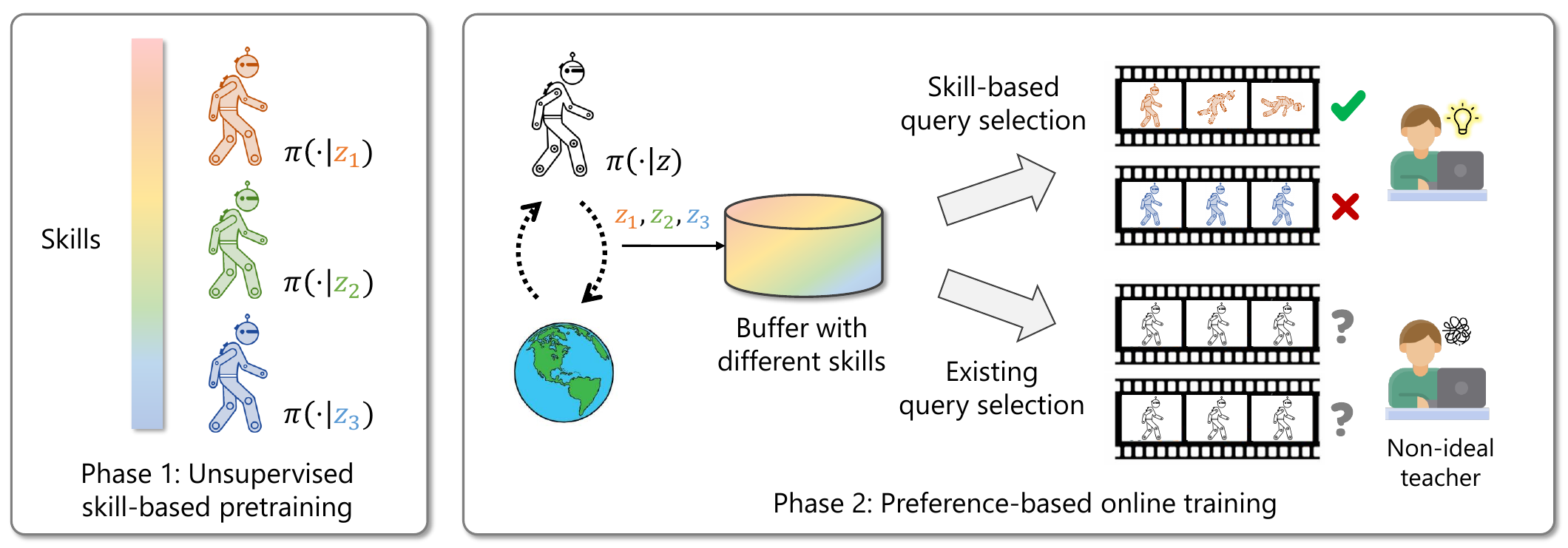}
    \caption{
    The framework of \method. 
    In the pretraining phase, we learn diverse skills based on unsupervised skill discovery methods. 
    In the online training phase, we leverage a novel skill-based query selection method to generate distinguishable queries for non-ideal teachers.
    }
    \label{fig:sepoa_framework}
\end{figure*}

\section{Introduction}
Reinforcement Learning (RL) has made significant progress across a variety of fields, including gameplay~\cite{mnih2013playing,silver2016mastering}, robotics~\cite{chen2022towards} and autonomous systems~\cite{bellemare2020autonomous,mu2024large,luan2025efficient}.
Yet, the success of RL often relies on the careful construction of reward functions, a process that can be both labor-intensive and costly.
To solve this issue, Preference-based Reinforcement Learning (PbRL) emerges as a compelling alternative~\cite{christiano2017deep,lee2021pebble}.
PbRL uses human-provided preferences among various agent behaviors to serve as the reward signal, thereby eliminating the need for hand-crafted reward functions.

Existing PbRL methods~\cite{lee2021pebble,park2022surf,shin2023benchmarks,kim2023preference} focus on enhancing feedback efficiency, by maximizing the expected return with minimal feedback queries.
However, these methods rely on high-quality or even ideal expert feedback, overlooking an important issue in preference labeling: \textbf{indistinguishability of segments}.
For example, asking humans to specify preferences between two similar trajectories can be challenging, as it is difficult for human observers to discern which is superior. 
Consequently, the resulting preference labels are often incorrect, further degrading the performance of the algorithm~\cite{lee2021b}.
Therefore, the segment indistinguishability issue limits the broader applicability of PbRL.

In the field of unsupervised reinforcement learning, information-theoretic methods have been shown to discover useful and diverse skills, without the need for rewards~\cite{eysenbach2019diversity,hansen2019fast}.
In contrast to the unlabeled agent behaviors in PbRL, these discovered skills possess a higher degree of distinguishability, allowing humans to easily express preferences between different skills.
However, the skill-driven approach has not been fully explored in PbRL, and how to apply the discovered skills to preference learning remains unclear.
This naturally leads to the following question: 
\begin{center}
{{\it How can we integrate the skill mechanism with PbRL to overcome the indistinguishability of segments?}}
\end{center}

In this work, we aim to provide an effective solution to the important and practical issue in PbRL: indistinguishability of segments, by incorporating the skill mechanism.
First, we conduct skill-based unsupervised pretraining to learn useful and diverse skills.
Next, we introduce a novel query selection mechanism in the learned skill space, which effectively balances the information gain with the distinguishability of the query.
We name our method as \textbf{S}kill-\textbf{E}nhanced \textbf{P}reference \textbf{O}ptimization \textbf{A}lgorithm (\method).
Our experiments demonstrate the necessity of the above two techniques, showing that \method significantly outperforms conventional PbRL methods in terms of both robustness and learning efficiency.

In summary, our contributions are threefold: 
\begin{itemize}
\item First, we highlight the critical issue of segment indistinguishability, validate its practical significance through human experiments, and theoretically analyze the limitations of current mainstream query selection methods, such as disagreement.  
\item Second, we propose \method, a skill-driven reward learning framework, which selects queries in a highly distinguishable skill space to address the segment indistinguishability issue.
\item Lastly, we conduct extensive experiments to show that \method significantly outperforms conventional PbRL methods under non-ideal feedback conditions. 
The results indicate that by introducing the skill mechanism, we can effectively mitigate the segment indistinguishability issue, 
thereby broadening the application of PbRL.
\end{itemize}

%% file: text/3_preliminaries.tex
\section{Preliminaries}
\label{sec:preliminaries}

\paragraph{Reinforcement Learning.}
A Markov Decision Problem (MDP) could be characterized by the tuple $(\mathcal{S},\mathcal{A},P,r,\gamma)$, where $\mathcal{S}$ is the state space, $\mathcal{A}$ is the action space, $P:\mathcal{S}\times\mathcal{A}\rightarrow \Delta(\mathcal{S})$ is the transition function, 
$r:\mathcal{S}\times\mathcal{A}\rightarrow \mathbb{R}$ is the reward function, and $\gamma\in[0,1)$ is the discount factor balancing instant and future rewards.
A policy $\pi$ interacts with the environment by sampling action $a$ from distribution $\pi(s,a)$ when observing state $s$.
The goal of RL agent is to learn a policy $\pi: \mathcal{S}\rightarrow\Delta{(\mathcal{A})}$, which maximizes the expectation of a discounted cumulative reward: $\mathcal{L}(\pi)=\mathbb{E}_{\mu_0,\pi}\left[\sum_{t=0}^{\infty}\gamma^t r(s_t, a_t)\right]$.

For any policy $\pi$, the corresponding state-action value function is $Q^{\pi}(s,a)=\mathbb{E}[\sum_{k=0}^{\infty}\gamma^k r_{t+k}|S_t=s, A_t=a, \pi]$.
The state value function is $V^{\pi}(s)=\mathbb{E}[\sum_{k=0}^{\infty}\gamma^k r_{t+k}|S_t=s, \pi]$.
It follows from the Bellman equation that $V^{\pi}(s)=\sum_{a\in\mathcal{A}}\pi(a|s)Q^{\pi}(s,a)$.

\paragraph{Preference-based Reinforcement Learning.}
In PbRL, the reward function $r$ is replaced by human-provided preferences over segment pairs, denoted as $\sigma_0, \sigma_1$.
A segment $\sigma$ is a continuous sequence in a fixed length $H$ of states and actions, i.e. $\{s_k,a_k, \dots,s_{k+H-1},a_{k+H-1}\}$.
Preferences are expressed as one-hot labels $y\in\{(0,1),(1,0)\}$, indicating the preferred segment.
All preference triples $(\sigma_0, \sigma_1, y)$ are stored in the dataset $D$.
The algorithm first estimates a reward function $\hat r_\psi:\mathcal{S}\times\mathcal{A}\rightarrow \mathbb{R}$, parameterized by $\psi$, using provided preferences,
This learned reward function $\hat{r}_\psi$ is then used to train the policy with standard RL algorithms.
To construct $\hat{r}_\psi$, we employ the Bradley-Terry model \cite{bradley-terry,christiano2017deep} as follows:
\begin{equation}
\label{eq:BT_model}
    P_\psi[\sigma_1\succ\sigma_0] =
    \frac{\exp\sum_t\hat{r}_\psi({s}_t^1,{a}_t^1)}{\sum_{i\in\{0,1\}}\exp\sum_t\hat{r}_\psi({s}_t^i,{a}_t^i)}.
\end{equation}
where $\sigma_1\succ \sigma_0$ indicates the human prefer $\sigma_1$ than $\sigma_0$.
$\hat r_\psi$ can be trained by minimizing the cross-entropy loss:
\begin{equation}
\label{eq:CE_loss}
    \begin{aligned}
        \mathcal{L}_{\mathrm{reward}}(\psi) =
        -&\underset{(\sigma_0,\sigma_1,y) \sim \mathcal{D}}{\mathbb{E}}
        \Big[ y(0) \log P_\psi[\sigma_0\succ\sigma_1]          \\
              & + y(1) \log P_\psi[\sigma_1\succ\sigma_0] \Big].
    \end{aligned}
\end{equation}

%% file: text/4_gap.tex
\section{Indistinguishability of Segments}

\label{sec:gap}

In an ideal PbRL scenario, human labelers can give the true preference between two distinct behaviors. 
In practice, however, humans often have to label similar trajectories, which can lead to mistakes. These errors could reduce the precision of the trained reward function, and consequently degrade the performance.
We name this labeling issue as the \textbf{Indistinguishability of Segments}. 
The issue significantly limits the broad application of PbRL, particularly in safety control fields~\cite{fulton2018safe} where the precision of the reward function is strictly required.

\paragraph{Human experiments. }
To validate this issue, we conduct human experiments, where human labelers provide preferences between segments with various return differences.
Then, we calculate the match ratio between human labels and ground truth.
Specifically, the human labelers watch a video rendering each segment and select the one that better achieves the objective based on the instruction. 
For example, the instruction given to human teachers in the $\verb|Cheetah_run|$ task is to run as fast as possible.
The results in Figure \ref{fig:human_get_confused} shows that as the return differences decrease, the match rate between human-labeled and ground truth preferences diminishes, indicating an increase in labeling errors. 
This confirms the practical significance of the indistinguishability issue.
Please refer to Appendix~\ref{app:human_exp} for experimental details.


\begin{figure}
    \centering
    \includegraphics[width=0.9\linewidth]{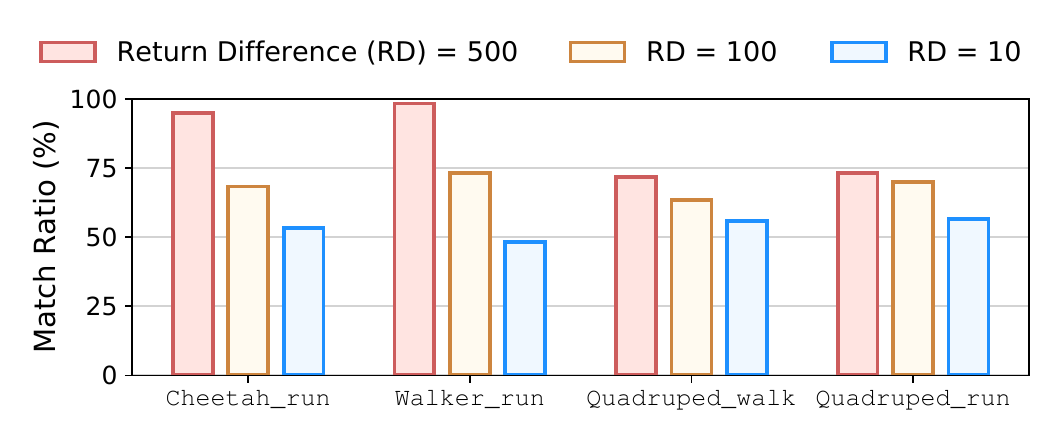}
    \caption{Human-labeled preferences match ratio with ground truth. As the return differences decrease, labeling errors increase.}
    \label{fig:human_get_confused}
\end{figure}

Furthermore, the segment indistinguishability issue can be particularly severe when the query selection method prioritizes ``informative'' queries. 
A representative example is the disagreement mechanism in PEBBLE \cite{lee2021pebble}, which selects queries based on the highest variance in predictions from an ensemble of reward models. 
To illustrate why this occurs, we conduct the following theoretical analysis.
As shown in Proposition \ref{thm:disagreement-bad}, the disagreement mechanism tends to select queries with similar return values, resulting in segments with similar and often indistinguishable behaviors.

\begin{restatable}{proposition}{disagreementbad}
\label{thm:disagreement-bad}
Let $\{\hat{r}^i\}$ be an ensemble of i.i.d. reward estimators, and $(\sigma_1,\sigma_2)$ be a segment pair with ground-truth cumulative
discounted reward $r_1\ge r_2$.
Suppose $\hat r^i$ estimates the cumulative discounted reward of $\sigma_j$ as $\hat r^i_j\sim N(r_j, c)$ ($c$ is a constant), and induces preference
\begin{equation}
    \label{eq:sigmoid-perf}
    \hat P_i[\sigma_1\succ\sigma_2] =
    \frac{\exp\hat{r}^i_1}{\exp\hat{r}^i_1+\exp\hat{r}^i_2}=\mathrm{sigmoid}(\hat r^i_1-\hat r^i_2).
\end{equation}
Then the disagreement of induced preference across $\{\hat{r}^i\}$, i.e. $\mathrm{Var}[\hat P[\sigma_1\succ\sigma_2]]$, approximately and monotonically increases as the dissimilarity of segment pair $\Delta =r_1-r_2$ decreases.
\end{restatable}






%% file: text/5_method.tex
\begin{algorithm}[t]
\caption{\method Framework}
\label{alg:sepoa_full_version}
\begin{algorithmic}[1]
\REQUIRE frequency of feedback $K$, number of queries $M$ per feedback session, total feedback number $N_{\rm total}$
\STATE Initialize $\pi(a|s,z) \gets $ \textsc{Unsupervised Pretrain}
\FOR {each iteration}
    \IF {iteration \% $K$ == 0 and total feedback $ < N_{\rm total}$}
        \STATE Update trajectory estimator $R_{\theta}(z)$ based on Eq.~\ref{eq: estimate}
        \FOR {$m$ in $1\cdots M$}
                        \STATE $(\sigma_0, \sigma_1) \sim$ \textsc{Query Selection}
            \STATE Query the teacher for preference label $y$
            \STATE Store preference $\mathcal{D}\gets  
            \mathcal{D}\cup \{(\sigma_0, \sigma_1, y)\}$
        \ENDFOR
        \STATE Update the reward model $\hat r_{\psi}$ using $\mathcal{D}$
        \STATE Relabel replay buffer $\mathcal{B}$ using $\hat r_{\psi}$
    \ENDIF
    \IF {the current episode ends}
        \STATE Update $z_\text{task}$ based on Eq.~\ref{eq:z_task_maxR}
    \ENDIF
    \STATE Obtain action $a_t\sim \pi(a|s_t, z_\text{task})$ and next state $s_{t+1}$ 
                \STATE Store transitions $\mathcal{B}\gets  \mathcal{B}\cup \{(s_t, a_t, s_{t+1},\hat r_{\psi}(s_t,a_t)\}$
    \STATE Sample transitions $(s,a,s',\hat r_{\psi})$ from $\mathcal{B}$
    \STATE Minimize $\mathcal{L}_\text{critic}$ and $\mathcal{L}_\text{actor}$ with $\hat{r}_{\psi}$
\ENDFOR
\end{algorithmic}
\end{algorithm}

\section{Skill-Driven PbRL}

To address the issue of segment indistinguishability, we hope to choose segments with different behaviors. 
Skill discovery methods can discover diverse skills, which align with this requirement. 
In this section, we propose the Skill-Enhanced Preference Optimization Algorithm~(\method), which leverages skill discovery techniques to enhance PbRL’s reward learning by selecting distinguishable queries. 
\method can be integrated with any skill discovery method, by introducing the following two key components:
\begin{itemize}
    \item Skill-based unsupervised pretraining, where the agent explores the environment and learns useful skills without supervision~(see Section~\ref{subsec:aps_pretrain}).
    \item Skill-based query selection, which can select more distinguishable queries based on the learned skill space~(see Section~\ref{subsec:sd_query_selection}).
\end{itemize}
We show the overall framework of \method in Figure~\ref{fig:sepoa_framework} and Algorithm~\ref{alg:sepoa_full_version}.

\subsection{Skill-based Unsupervised Pretraining}
\label{subsec:aps_pretrain}

Previous unsupervised pretraining methods \cite{lee2021pebble,park2022surf} aim to help the agent explore the state space by maximizing the state entropy $H(s)$. 
However, these methods often fail to produce clearly distinguishable behaviors, as the learned policies primarily focus on exploration, rather than diversity. 
As a result, the queries may be difficult for human teachers to distinguish.
Skill-based unsupervised pretraining addresses this issue by guiding the agent to discover diverse skills, resulting in more distinct behaviors, and enabling the selection of distinguishable queries in the early training process.

In skill discovery, the policy is in the form of $\pi(a|s,z)$, where $z\in \mathcal{Z}$ denotes the skill and $\mathcal{Z}$ represents the skill space.
To ensure that the policies generated by each skill have distinct behaviors, we aim to maximize the mutual information between the policy's behavior and its skill.
For tractability, we optimize a variational lower bound, as shown in Eq.~\ref{eq:skill_forward}, where $p$ is the underlying distribution, and $q_\phi$ is learned via maximum likelihood on data sampled from $p$.
\begin{equation}
\begin{aligned}
I(s;z) &= \mathbb{E}_{s,z\sim p(s,z)}[\log p(s|z)] - \mathbb{E}_{s\sim p(s)}[\log p(s)] \\
& \geq \mathbb{E}_{s,z\sim p(s,z)}[\log q_{\phi}(s|z)] - \mathbb{E}_{s\sim p(s)}[\log p(s)]
\end{aligned}
\label{eq:skill_forward}
\end{equation}
Skill discovery methods use intrinsic rewards to encourage agents to explore different behaviors. A typical form of the intrinsic reward is shown in Eq.~\ref{eq:skill_intrinsic_reward}:
\begin{equation}
r^\text{int}=\log q_\phi(s|z)-\log p(s).
\label{eq:skill_intrinsic_reward}
\end{equation}
In this way, after pre-training in this stage, we obtained a policy in the form of $\pi(a|s,z)$. By selecting different skills $z$, we can generate segments with diverse behaviors.

\begin{figure*}[ht]
    \centering
    \includegraphics[width=0.95\linewidth]{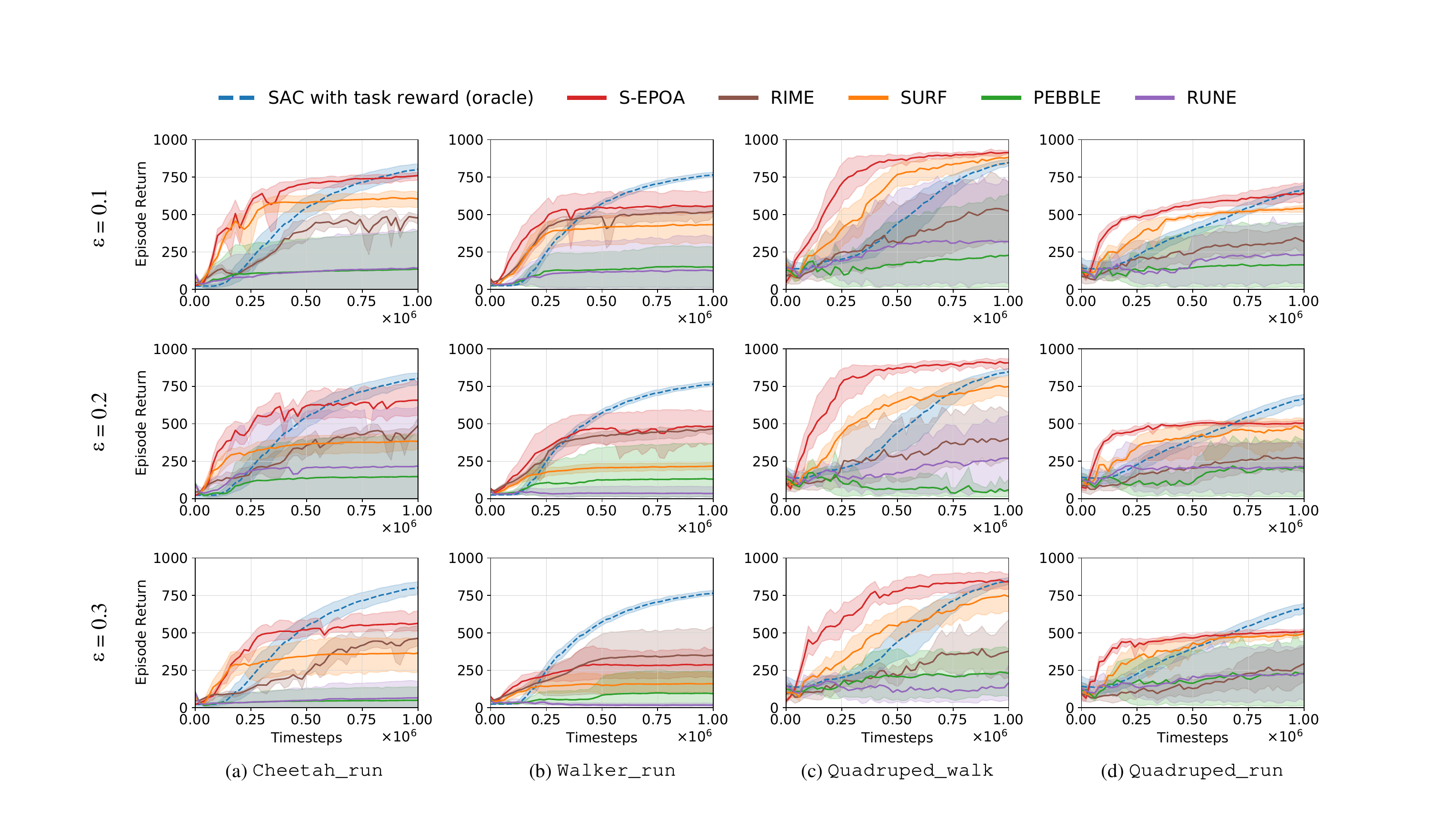}
    \caption{
    Learning curves on locomotion tasks from DMControl, where each row corresponds to a different error rate $\epsilon$, and each column represents a specific task. 
    SAC serves as an oracle, using the ground-truth reward unavailable in PbRL settings. 
        The solid line and shaded regions respectively represent the mean and standard deviation of episode return, across $5$ seeds.
        }
    \label{fig:all_dmcontrol}
\end{figure*}

\subsection{Skill-based Query Selection}
\label{subsec:sd_query_selection}

In this subsection, we focus on selecting distinguishable queries to enhance reward learning. 
To achieve this, the two segments being compared should come from different skills with distinct behaviors and performances. 
We introduce a trajectory estimator $R_{\theta}(z)$ to estimate the expected return of trajectories generated by skill $z$, which helps us identify skills with significant performance differences. The training objective of $R_{\theta}(z)$ is:
\begin{equation}
\begin{aligned}
    & \min \mathcal{L}_{\mathrm{est}}(\theta) = \mathbb{E}_{z}
        \bigg[
    R_{\theta}(z) - \mathbb{E}_{\tau \sim \pi(\cdot|\cdot, z)}\sum_{(s,a)\in\tau} \hat{r}_{\psi}(s,a) 
        \bigg]^2,
        \label{eq: estimate}
\end{aligned}
\end{equation}
where $\tau$ is the trajectory generated by $z$, and $\hat{r}_{\psi}$ is the learned reward model.
In practice, we normalize the targets of $R_{\theta}(z)$ to the range of $[0, 1]$ for training stability.
Based on $R_{\theta}(z)$, we define the skill-based selection criteria $I(\sigma_0, \sigma_1)$ for query $(\sigma_0, \sigma_1)$ with underlying skills $(z_0, z_1)$:
\begin{equation}
\label{eq:I}
\begin{aligned}
    I(\sigma_0, \sigma_1) = 
    (1+ |R_{\theta}(z_0) & - R_{\theta}(z_1)| ) \cdot \\
    & \left(1+ \text{Var}(P_\psi[\sigma_1 \succ \sigma_0]) \right) ,
\end{aligned}
\end{equation}
where $P_\psi$ is the probability that reward model prefers $\sigma_1$ than $\sigma_0$, as defined in Eq.~\ref{eq:BT_model}. 
The first term assesses the difference between the skills, 
while the second term measures the reward model uncertainty, which is commonly used in prior works \cite{lee2021pebble,liang2022reward}.
For training stability, we normalize both terms to the $[0, 1]$ range and add $1$ to balance the two values.

Thus, we propose the skill-based query selection method: 
For each query $(\sigma_0, \sigma_1)$, we calculate the skill-based selection criteria $I(\sigma_0,\sigma_1)$, and select queries with the highest value.
Based on Eq.~\ref{eq:I}, this approach not only considers the query's uncertainty to maximize the information gain, but also considers the differences between skills, ensuring the segments have clearly distinguishable skill explanations.
The specific method is illustrated in Algorithm \ref{alg:sdpbrl_query_selection}.

\begin{algorithm}[t]
\caption{\textsc{Query Selection}}
\label{alg:sdpbrl_query_selection}
\begin{algorithmic}[1]
\STATE Randomly sample $N$ queries $(\sigma_0, \sigma_1)$, where $\sigma_0, \sigma_1$ are segments generated by skills $z_0, z_1$
\STATE Calculate the difference $|R_{\theta}(z_0) - R_{\theta}(z_1)|$
\STATE Calculate the uncertainty $\text{Var}(P_\psi[\sigma_0 \succ \sigma_1])$
\STATE Normalize $|R_{\theta}(z_0) - R_{\theta}(z_1)|$ and $\text{Var}(P_\psi[\sigma_0 \succ \sigma_1])$
\STATE Calculate $I(\sigma_0, \sigma_1)$ for each query
\STATE Select the query with maximum $I(\sigma_0, \sigma_1)$
\end{algorithmic}
\end{algorithm}

\begin{figure*}[ht]
    \centering
        \includegraphics[width=0.75\linewidth]{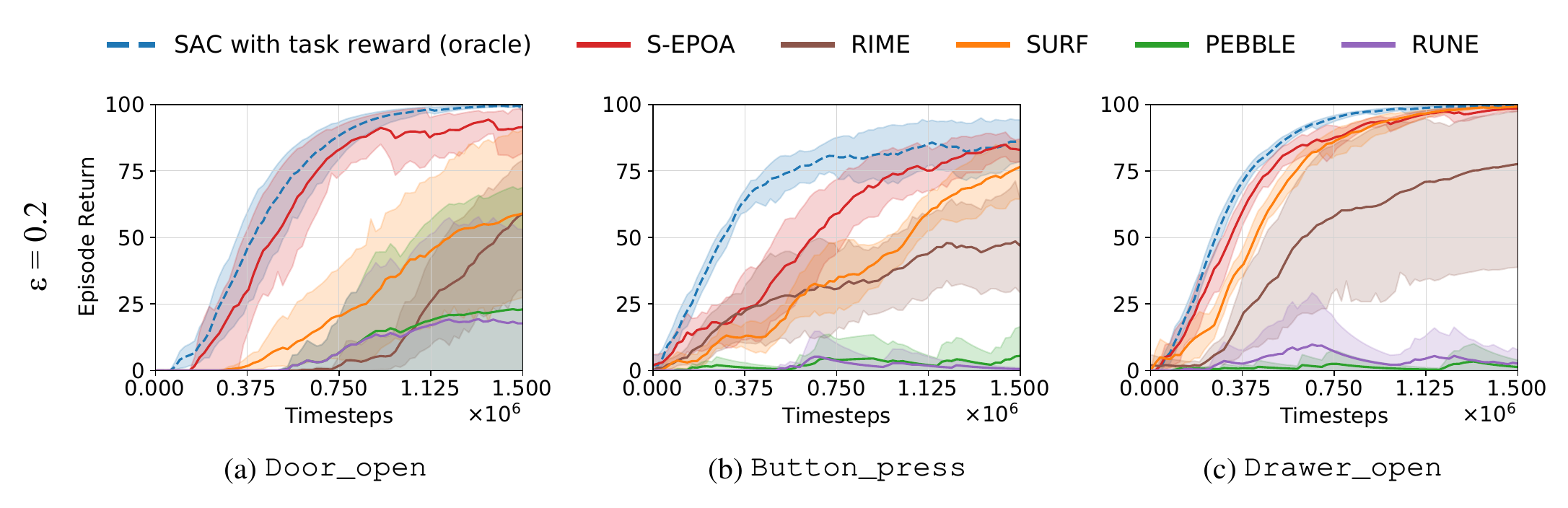}
    \caption{
    Learning curves on locomotion tasks from Metaworld, with error rate $\epsilon=0.2$, across $5$ seeds. 
    }
    \label{fig:all_metaworld}
\end{figure*}

\subsection{Implementation Details}

To convert the unsupervised pretraining policy $\pi(a|s,z)$ in Section~\ref{subsec:aps_pretrain} to PbRL's policy $\pi(a|s)$ in Section~\ref{subsec:sd_query_selection},
we attempt to obtain the skill nearest to the current task, i.e., performs the best under the current estimated reward function. This skill is denoted as $z_\text{task}$.
Intuitively, $z_\text{task}$ can be derived by solving the optimization problem in Eq.~\ref{eq:get_z_task}:
\begin{equation}
    z_\text{task} = \arg\max_z \mathbb{E}_{\mu_0,\pi(a_t|s_t,z)}\left[\sum_{t=0}^{\infty}\gamma^t \hat{r}(s_t, a_t)\right]
\label{eq:get_z_task}
\end{equation} 
Using the learned trajectory estimator $R_\theta(z)$, we sample $N_z$ skills $z_i$ from $\mathcal{Z}$ uniformly, and approximate $z_\text{task}$ with
\begin{equation}
    \hat z_\text{task} = \arg\max_{z_i} \{R_\theta(z_i)\}_{i=1}^{N_z}.
\label{eq:z_task_maxR}
\end{equation}

Based on the above discussion, we obtain the overall process of \method, as shown in Algorithm~\ref{alg:sepoa_full_version}.
Firstly, we initialize the policy with skill-based unsupervised pretraining.
Then, for each feedback session, we update the trajectory estimator as in Eq.~\ref{eq: estimate}, and select queries based on the skill-based selection criteria in Eq.~\ref{eq:I}. 
The PbRL reward model $\hat r_\psi$ is trained using the selected queries based on Eq.~\ref{eq:CE_loss}.
Finally, we train the critic and actor using the learned reward function $\hat{r}_{\psi}$.

Besides the two key components in Section \ref{subsec:aps_pretrain} and \ref{subsec:sd_query_selection}, we also adopt the semi-supervised data augmentation technique for reward learning \cite{park2022surf}. 
To elaborate, we randomly sub-sample several shorter pairs of $(\hat \sigma_0, \hat \sigma_1)$ from the queried segments $(\sigma_0, \sigma_1, y)$, and use these $(\hat \sigma_0, \hat \sigma_1, y)$ to optimize the cross-entropy loss in Eq. \ref{eq:CE_loss}. 
Moreover, we sample a batch of unlabeled segments $(\sigma_0, \sigma_1)$, generate the artificial labels $\hat y$, if $P_\psi[\sigma_0\succ \sigma_1]$ or $P_\psi[\sigma_1\succ \sigma_0]$ reaches a predefined confidence threshold.

%% file: text/6_experiments.tex
\begin{figure*}[ht]
    \centering
    \subfloat[\label{subfig:component_analysis}]{
        \includegraphics[width=0.21\linewidth]{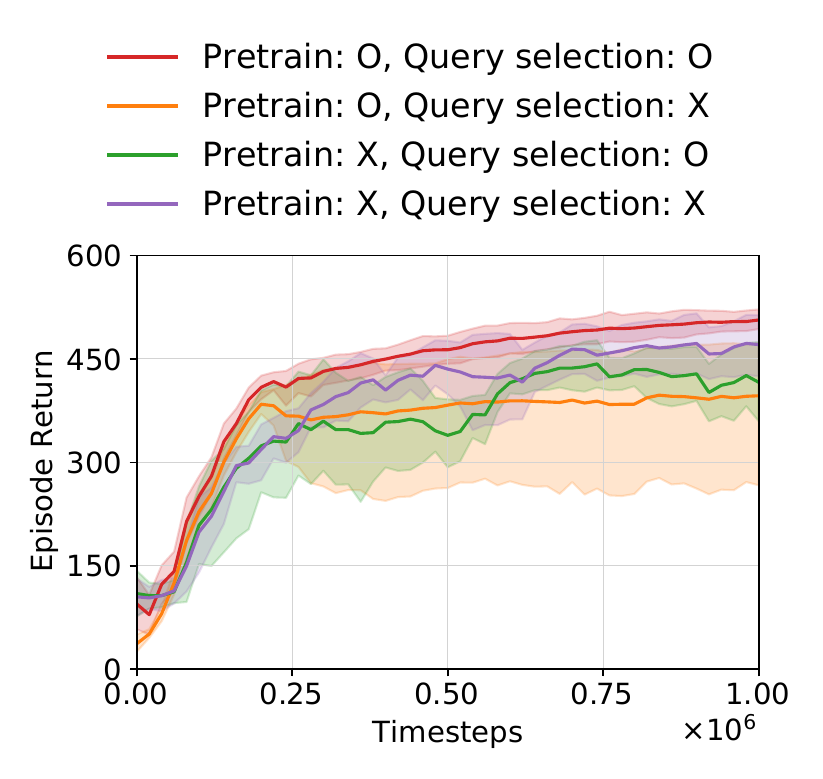}}
    \subfloat[\label{subfig:le_error_0}]{
        \includegraphics[width=0.21\linewidth]{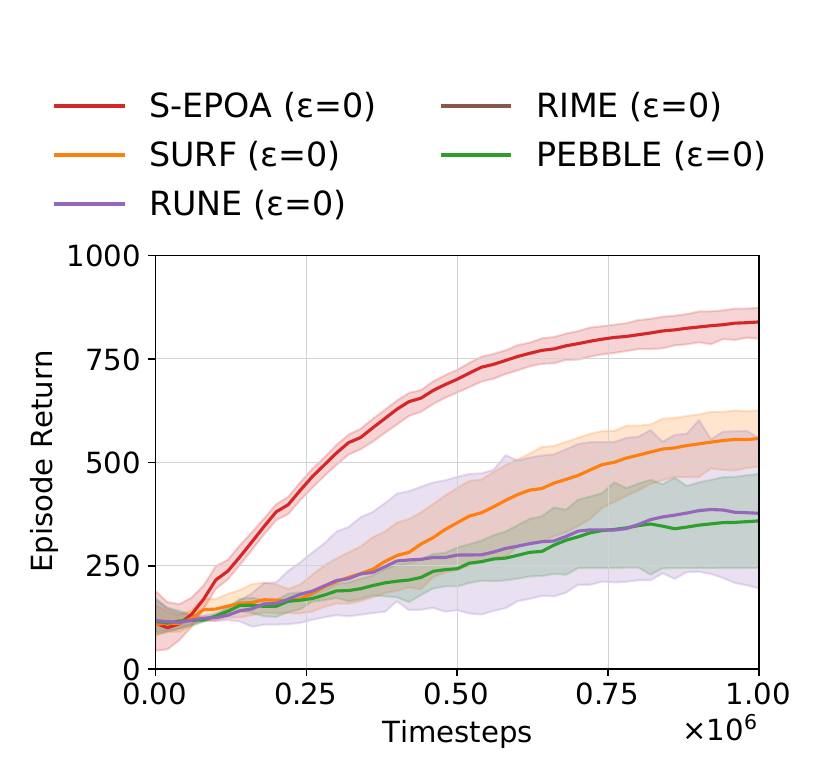}}
    \subfloat[\label{subfig:other_skills}]{
        \includegraphics[width=0.21\linewidth]{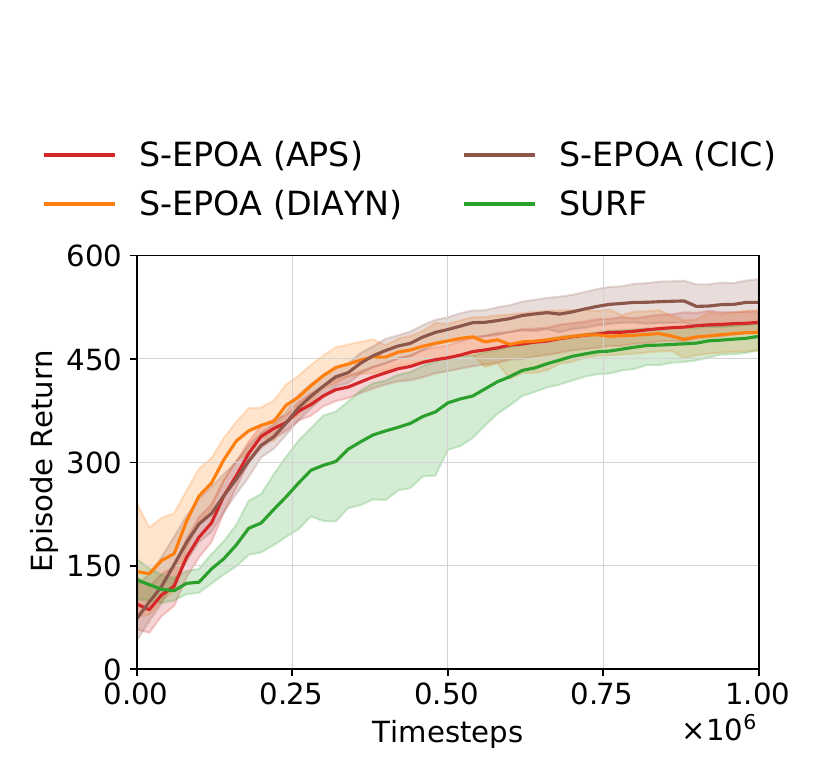}}
    \subfloat[\label{subfig:all_nosurf}]{
        \includegraphics[width=0.21\linewidth]{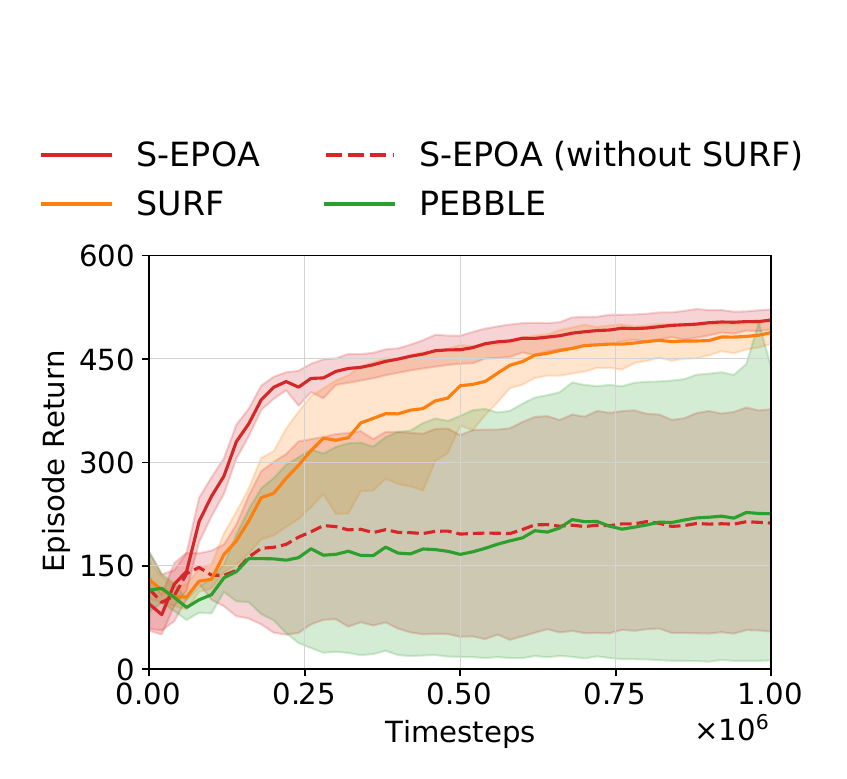}}
    \cprotect\caption{%
    Ablation studies on the \verb|Quadruped_run| task.
    (a) Contribution of each technique in \method, under $\epsilon=0.3$.
    (b) Demonstration of enhanced learning efficiency of \method under the ideal scripted teacher with error rate $\epsilon=0$.
    (c) The learning curve of \method and baselines, with and without data augmentation under $\epsilon=0.3$.
    (d) Integrating \method with other skill discovery methods, under $\epsilon=0.3$.
    }
    \label{fig:ablation}
\end{figure*}

\section{Experiments}

We design our experiments to answer the following questions:
\textit{Q1:} How does \method compare to other state-of-the-art methods under non-ideal teachers?
\textit{Q2:} Can \method select queries with higher distinguishability?
\textit{Q3:} Can \method be integrated with various skill discovery methods? 
\textit{Q4:} What is the contribution of each of the proposed techniques in \method?

\subsection{Setup}

\paragraph{Domains. }
We evaluate \method on several complex robotic manipulation and locomotion tasks from DMControl \cite{tassa2018dmcontrol} and Metaworld \cite{yu2020metaworld}. 
Specifically, We choose $4$ complex tasks in DMControl: \verb|Cheetah_run|, \verb|Walker_run|, \verb|Quadruped_walk|, \verb|Quadruped_run|, and $3$ complex tasks in Metaworld:  \verb|Door_open|, \verb|Button_press|, \verb|Window_open|.
The details of experimental tasks are shown in Appendix \ref{app:tasks}.

\paragraph{Baselines and Implementation. }
We compare \method with several state-of-the-art methods, including PEBBLE~\cite{lee2021pebble}, SURF~\cite{park2022surf}, RUNE~\cite{liang2022reward}, and RIME~\cite{cheng2024rime}, a robust PbRL method. We also train SAC with ground truth rewards as a performance upper bound.
For PEBBLE, SURF, and RUNE, we employ the disagreement query selection, which performs the best among all the query selection methods.
In our experiment, we use APS \cite{liu2021aps} for unsupervised skill discovery. The impact of skill discovery methods is discussed in the ablation study.
More implementation details are provided in Appendix \ref{app:implement_detail} and \ref{app:skill_discovery}.

\paragraph{Noisy scripted teacher imitating humans.}
Following prior works \cite{lee2021pebble,liang2022reward},  we use a scripted teacher for systematic evaluation, which provides preferences between segments based on the sum of ground-truth rewards. 
To better mimic human decision-making uncertainty, we introduce a noisy scripted teacher.
When the performance difference between two policies is marginal, humans often struggle to distinguish between them, as Section \ref{sec:gap} shows.
To imitate this, we implement an error mechanism: if the ground truth returns of two trajectories are nearly identical, we randomly assign a preference label of $0$ or $1$. 
The core idea is to evaluate policy performance by comparing the overall returns of entire trajectories, which aligns with how humans assess policies based on their overall effectiveness.
Specifically, for a query $(\sigma_0, \sigma_1)$ with underlying trajectories $(\tau_0, \tau_1)$ and ground truth reward function $r_\text{gt}$, if 
\begin{equation}
    \bigg|
    \begin{aligned}
    \sum_{(s,a)\in\tau_0} r_\text{gt}(s,a) - \sum_{(s,a)\in\tau_1} r_\text{gt}(s,a)
    \end{aligned}
    \bigg|
    < \epsilon \cdot R_\text{avg} ,
\end{equation}
a random label is assigned. 
Here, $R_\text{avg}$ is the average return of the most recent ten trajectories. 
We refer to $\epsilon\in(0,1)$ as the error rate. 
For fairness, we constrain that the two segments in a query come from different trajectories.
It is important to note that our noisy teacher differs from the ``mistake'' teacher in B-Pref \cite{lee2021b}, which randomly flips correct labels. Our teacher only introduces errors in too-close queries.

\begin{figure}
    \centering
    \includegraphics[width=0.9\linewidth]{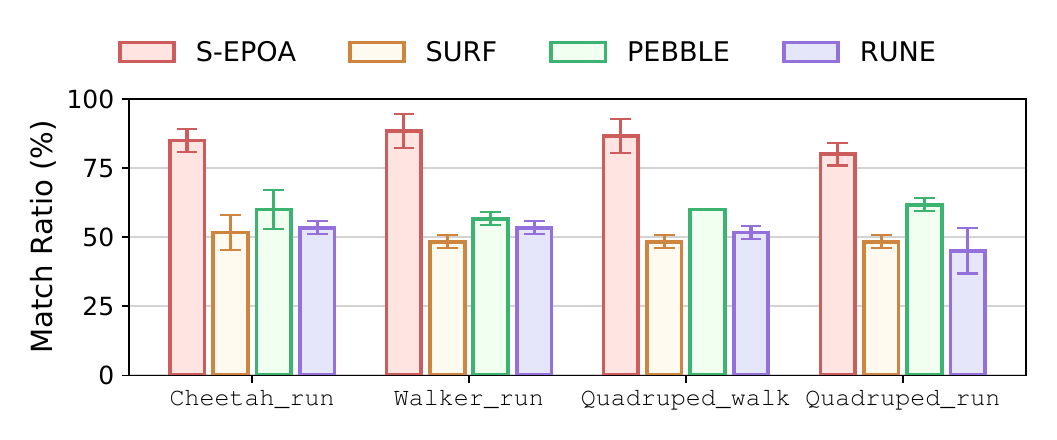}
    \caption{Human-labeled preferences match ratio with ground truth. Queries selected by \method are more distinguishable than others.}
    \label{fig:human_like_sepoa}
\end{figure}

\begin{figure*}[t]
    \centering
    \subfloat[\label{fig:segment_disagreement}]{\includegraphics[width=0.4\linewidth]{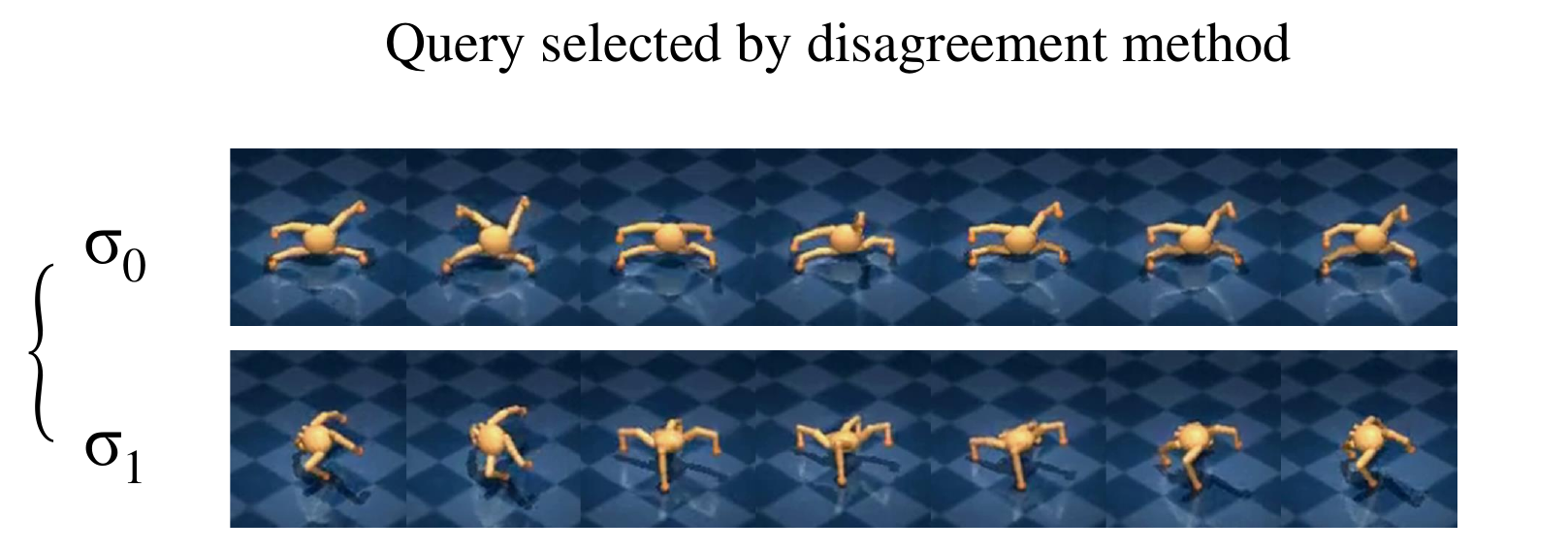}}
    \subfloat[\label{fig:c}]{\includegraphics[width=0.4\linewidth]{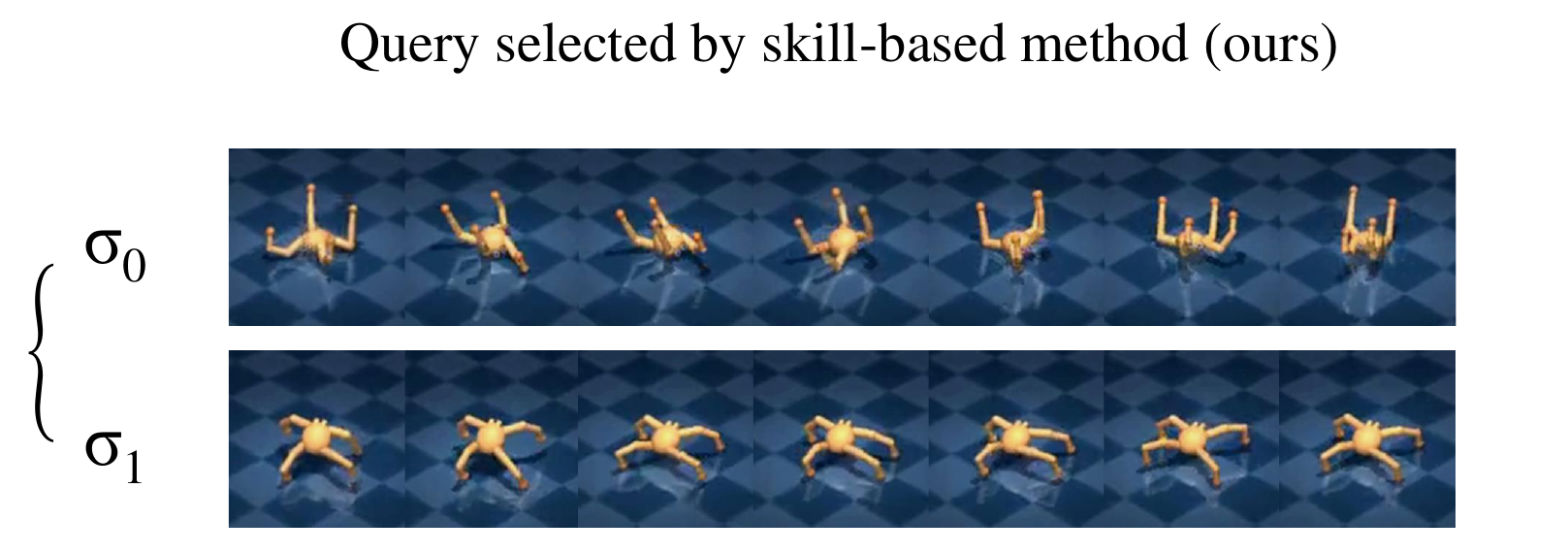}}
    \cprotect\caption{%
    Visualization of segment pairs selected by (a) disagreement mechanism and (b) skill-based mechanism, under the \verb|Quadruped_run| task, with error rate $\epsilon=0.3$. 
    }
    \label{fig:skill_visualization}
\end{figure*}

\subsection{Results on Benchmark Tasks}

\paragraph{Locomotion tasks in DMControl. }
Figure \ref{fig:all_dmcontrol} shows the learning curves of \method and baselines on the four DMControl tasks with error rates $\epsilon \in \{0.1, 0.2, 0.3\}$. \method outperforms baselines in most environments and remains robust under non-ideal conditions, while other PbRL methods are unstable and even fail.

\paragraph{Robotic manipulation tasks in Metaworld. }
Figure \ref{fig:all_metaworld} shows the learning curves for \method and baselines on the three Metaworld tasks with $\epsilon=0.2$. These results further demonstrate that \method improves robustness against non-ideal feedback across diverse tasks.

\paragraph{Enhanced query distinguishability. }
To evaluate that \method selects more distinguishable queries, we compare the ratios of queries that can be distinguished by the noisy teacher, for both \method and disagreement query selection methods. As shown in Table~\ref{tab:error_ratio}, \method achieves higher distinguishability ratios across all environments, demonstrating its effectiveness in selecting easily differentiable queries. 
To further assess the distinguishability of queries, we conduct human experiments. We compare \method with three state-of-the-art PbRL methods in DMControl tasks, where labelers provide $20$ preference labels for each run, across $3$ random seeds. As shown in Figure \ref{fig:human_like_sepoa}, humans find the queries selected by \method easier to distinguish. Please refer to Appendix \ref{app:human_exp} for experimental details.

\paragraph{Query visualizations.}
We visualize the segment pairs selected by both \method and the disagreement mechanism. As shown in Figure \ref{fig:skill_visualization}, the pair chosen by the disagreement mechanism has similar behaviors, while the pair selected by \method clearly differs, with $\sigma_1$ being preferred. 
These results confirm that \method can select more distinguishable queries, enhancing both the labeling accuracy of noisy teachers and the quality of human preference labeling.

\subsection{Ablation Study}
\label{subsec:ablation}

\paragraph{Component analysis.} 

To evaluate the effect of each technique in \method individually, we incrementally apply skill-based unsupervised pretraining and skill-based query selection to our backbone algorithm. 
Figure~\ref{fig:ablation}\subref{subfig:component_analysis} shows the learning curves on the \verb|Quadruped_run| task with error rate $\epsilon=0.3$. 
First, skill-based unsupervised pretraining significantly boosts performance, for both skill-based (red vs. green) and disagreement-based query selection (orange vs. blue). 
This improvement is because the pretrained policy generates diverse behaviors, which leads to a better reward function.
Additionally, skill-based query selection enhances results (red vs.~orange), as it selects more distinguishable queries, enabling the agent to obtain more accurate labels.
In summary, both components of \method are effective, and their combination is crucial to the method’s success.


\begin{table}[t]
    \centering
    \begin{tabular}{c|cc}
    \toprule
         & Disagreement & Skill-based \\
    \midrule
    \verb|Cheetah_run|     & 0.3270 & \textbf{0.4839} \\
    \verb|Walker_run|      & 0.2448 & \textbf{0.4648} \\
    \verb|Quadruped_walk|  & 0.3570 & \textbf{0.3800} \\
    \verb|Quadruped_run|   & 0.2545 & \textbf{0.2856} \\
    \bottomrule
    \end{tabular}
    \caption{Ratios of queries that can be distinguished by the noisy scripted teacher for both skill-based and disagreement query selection methods, with $\epsilon=0.3$.}
    \label{tab:error_ratio}
\end{table}

\begin{table}[t]
\centering
\begin{tabular}{ccc}
\toprule
\# of   Queries & \method         & SURF            \\
\midrule
500             & 444.25 ± 38.94  & 403.85 ± 116.18 \\
1000            & 464.36 ± 37.04  & 438.29 ± 95.56  \\
2000            & 506.57 ± 19.24  & 488.33 ± 21.78  \\
10000           & 568.86 ± 197.01 & 468.49 ± 22.92  \\
\bottomrule
\end{tabular}
\caption{Performance of \method and SURF using different numbers of queries, under $\epsilon=0.3$.}
\label{tab:sample_efficiency}
\end{table}




\paragraph{Enhanced learning efficiency under the ideal teacher. }
Under the ideal scripted teacher ($\epsilon = 0$), \method can also significantly enhance learning efficiency. 
As shown in Figure~\ref{fig:ablation}\subref{subfig:le_error_0}, \method converges faster and achieves better final performance. 
This advantage results from selecting skills with high distinguishability, highlighting the robustness and effectiveness of our approach.

\paragraph{Integrate \method with other skill discovery methods. }
To show that \method is integratable with various skill learning methods, we replace APS with DIAYN \cite{eysenbach2019diversity} and CIC \cite{laskin2022cic}, which are typical and commonly used skill discovery methods.
The details of these methods are shown in Appendix \ref{app:skill_discovery}.
As is shown in Figure~\ref{fig:ablation}\subref{subfig:other_skills}, the performance is consistent and similar across all skill discovery methods, which demonstrates that \method is compatible with various skill discovery methods.

\paragraph{Enhanced query efficiency. }
We compare the performance of \method and SURF using different numbers of queries.
As is shown in Table \ref{tab:sample_efficiency}, \method outperforms SURF consistently, even if only 500 queries are provided, which demonstrates the ability of \method to make better use of the limited queries.

\paragraph{Effect of data augmentation in \method. }
We conduct an ablation study to evaluate the impact of data augmentation in \method. 
Figure \ref{fig:ablation}\subref{subfig:all_nosurf} shows that without data augmentation, both \method and the baselines (PEBBLE as the backbone) show similar performance deficiencies. 
This highlights the importance of data augmentation for achieving optimal performance, while it is not the sole factor in our method’s success and does not undermine the innovation of our approach.

%% file: text/2_related_work.tex
\section{Related Work}

\paragraph{Preference-based reinforcement learning.}
PbRL enables humans (or supervisors in other forms, like script teachers) to guide the RL agent toward desired behaviors by providing preferences on segment pairs, where the feedback efficiency is a primary concern \cite{lee2021pebble,park2022surf}.
Prior works improve the feedback efficiency from various perspectives.
Some works focus on the query selection scheme, trying to improve the information quality of queries \cite{ibarz2018reward,biyik2020active,hejna2023few}.
Some works integrate unsupervised pretraining to avoid the waste on initial nonsense queries \cite{lee2021pebble}.
Some works augment queries from humans to better utilize limited human feedback \cite{park2022surf,liu2023efficient}.
These methods depend on reliable feedback.
However, humans could make mistakes, especially when the segment pair for comparison is slightly different, which restricts and even harms the performance in practice \cite{lee2021b,cheng2024rime}.

\paragraph{Unsupervised pretraining for RL.}
Unsupervised pretraining has been well studied in RL \cite{xie2022pretraining}, 
where unlabeled data (i.e., transitions without task-specific rewards) are used to learn a policy or set of policies that effectively explore the state space through intrinsic rewards. 
The method to calculate the intrinsic reward varies in different unsupervised pretraining works, including uncertainty measures like prediction errors \cite{pathak2017curiosity,pathak2019self,burda2019exploration}, state entropy \cite{hazan2019provably,liu2021behavior}, pseudo-counts \cite{bellemare2016unifying,ostrovski2017count} and empowerment measures like mutual information \cite{eysenbach2019diversity,sharma2020dynamics,liu2021aps,park2022lipschitz,park2023metra}.
The learned policy could serve as a strong initialization for downstream tasks, enhancing the sample efficiency in multi-task and few-shot RL.

\paragraph{Unsupervised skill discovery methods.}
Unsupervised skill discovery methods are a subset of unsupervised pretraining methods, which use empowerment measures as the intrinsic reward, trying to discover a set of distinguishable primitives.
Mutual information $I(s,z)$ is a common choice for the empowerment measure, where $s$ is a state, and $z$ is a latent variable indicating the skill. 
Some studies consider the reverse form $I(s,z)=H(z)-H(z|s)$~\cite{eysenbach2019diversity,park2022lipschitz}, which train a parameterized skill discriminator $q(z|s)$ together with the policy.
On the other hand, the forward form $I(s,z)=H(s)-H(s|z)$ \cite{sharma2020dynamics,liu2021aps,laskin2022cic} can be integrated with model-based RL and state entropy-based unsupervised pretraining algorithms. 
Additionally, some studies design the skill latent space for unique properties by parameterizing the distribution $q(z|s)$ or $q(s|z)$.
VISR \cite{hansen2019fast} and APS \cite{liu2021aps} let the latent $z$ be the successor feature to enable fast task inference.
LSD \cite{park2022lipschitz} and METRA \cite{park2023metra} bind the distance in state space and latent space to force a significant travel distance in a trajectory, thereby capturing dynamic skills.

%% file: text/7_conclusion.tex
\section{Conclusion}



This paper presents \method, a robust and efficient PbRL algorithm designed to address the segment indistinguishability issue. 
By leveraging skill mechanisms, \method learns diverse behaviors through unsupervised learning and generates distinguishable queries through skill-based query selection. 
Experiments show that \method outperforms state-of-the-art PbRL methods in robustness and efficiency, with ablation studies confirming the effectiveness of skill-based query selection. 
In future work, we aim to extend \method to broader applications.


%% file: appendix/1_proof.tex
\section {Proof}

\paragraph{Relationship between the similarity of segment pairs and disagreement.}
\label{sec:proof-disagreement}




\disagreementbad*

\begin{proof}
    Since $\hat r_1$ and $\hat r_2$ are independent Gaussian distributed, $\delta=\hat r_1-\hat r_2$ is also Gaussian distributed, i.e. $\delta \sim N(\Delta, \sqrt{2c^2})$.
    Substitute $\delta$ into Eq. \ref{eq:sigmoid-perf}, we find $P[\sigma_1\succ\sigma_2]$ is logit-normal distributed, whose moments have no analytic solution.
    
    As in \cite{huber2020bayesian}, we approximate the sigmoid function with the probit function using the input scaling factor $\lambda=\sqrt{\pi/8}$ \cite{kristiadi2020being}, leading to an approximation of $\mathrm{Var}[\hat P[\sigma_1\succ\sigma_2]]$:
    \begin{equation}
        \label{eq:pf2}
        \mathrm{Var}[\hat P[\sigma_1\succ\sigma_2]]\approx\mu_s(1-\mu_s)(1-\frac{1}{t}),
    \end{equation}
    where $\mu_s\approx\mathrm{sigmoid}(\Delta/t)\in[\frac{1}{2},1)$ is the approximation of $\mathbb{E}[\hat P[\sigma_1\succ\sigma_2]]$ derived in a similar manner, and $t=\sqrt{1+2\lambda^2c^2}$ is a constant \cite{huber2020bayesian}.
    
    Using Eq.\ref{eq:pf2}, it is straightforward to check the monotonicity of $\mathrm{Var}[\hat P[\sigma_1\succ\sigma_2]]$, which concludes the proof.
\end{proof}

%% file: appendix/2_experimental_details.tex
\section{Experimental Details}
\label{app:experiment_detail}

\subsection{Tasks}
\label{app:tasks}
The locomotion tasks from DMControl \cite{tassa2018dmcontrol} and robotic manipulation tasks from Metaworld \cite{yu2020metaworld} used in our experiments are shown in Figure \ref{fig:env examples}.

\begin{figure*}[ht]
\centering
\captionsetup[subfloat]{captionskip=-8pt}

\includegraphics[width=0.25\linewidth]{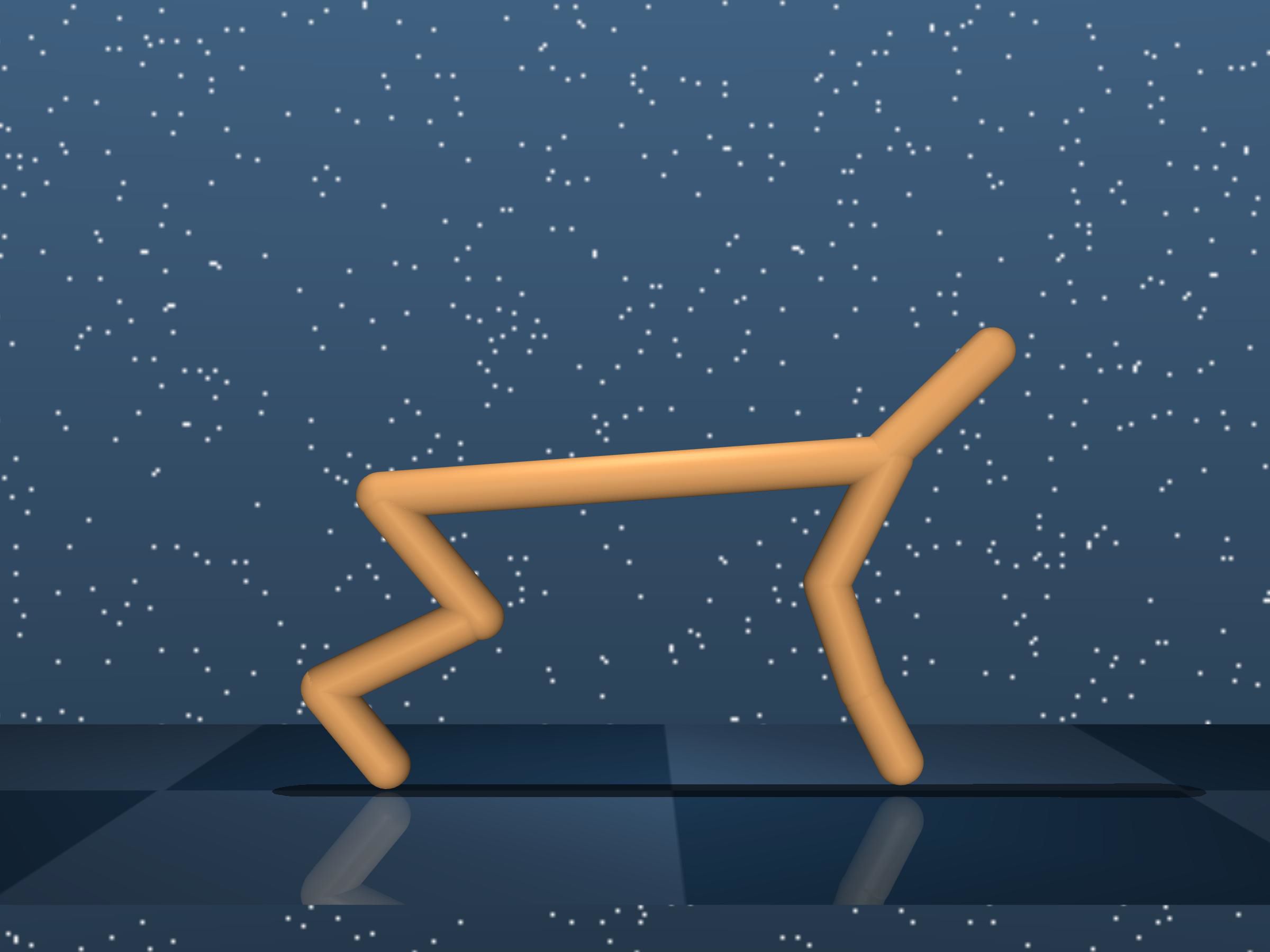}
\includegraphics[width=0.25\linewidth]{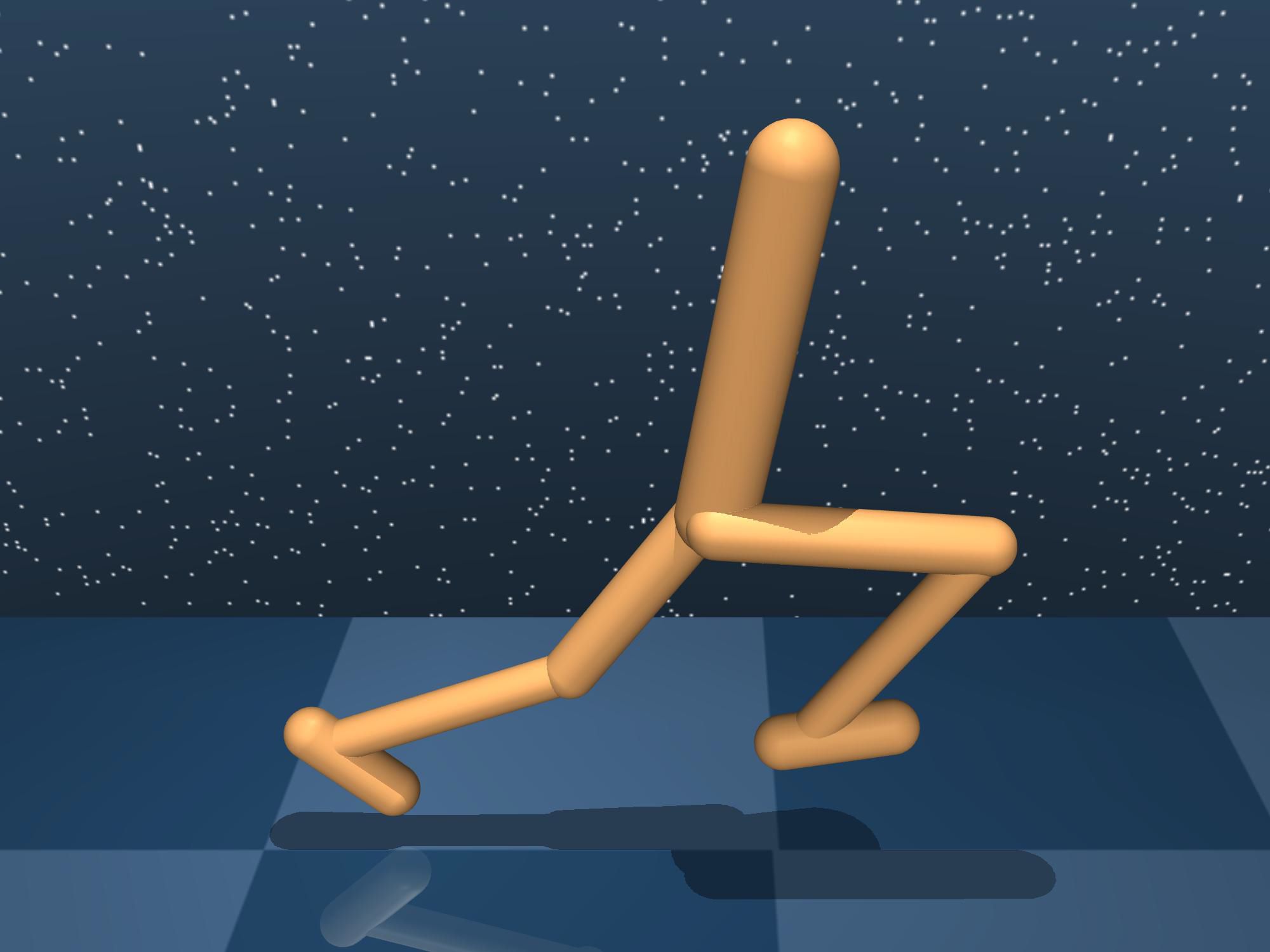}
\includegraphics[width=0.25\linewidth]{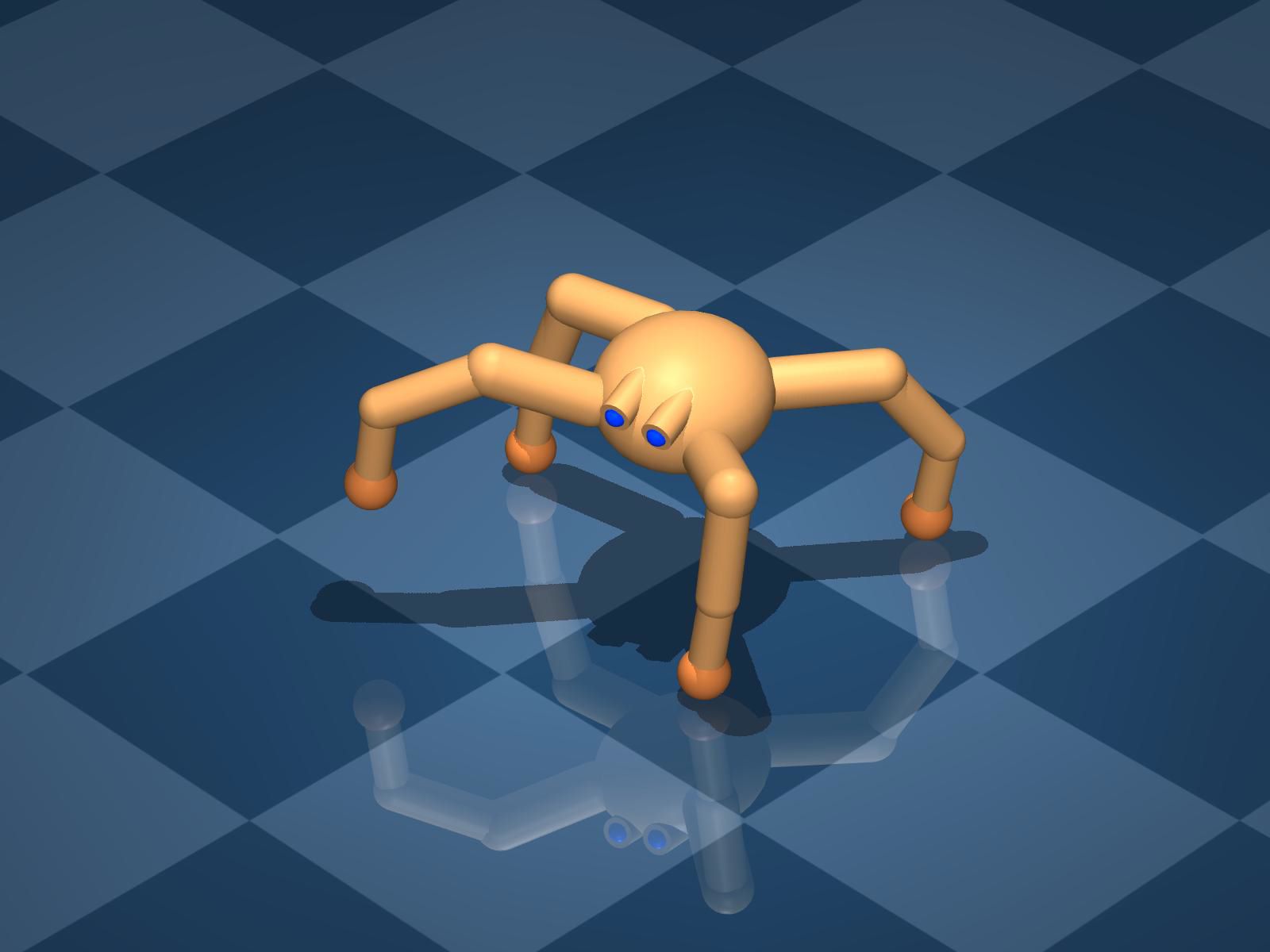}
\\
\subfloat[Cheetah\_run]{\hspace{0.256\linewidth}}
\subfloat[Walker\_run]{\hspace{0.256\linewidth}}
\subfloat[{\tiny Quadruped\_walk and Quadruped\_run}]{\hspace{0.256\linewidth}}
\\
\includegraphics[width=0.25\linewidth]{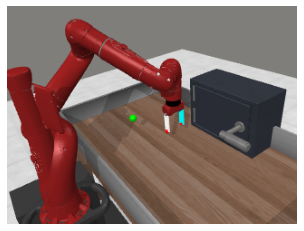}
\includegraphics[width=0.25\linewidth]{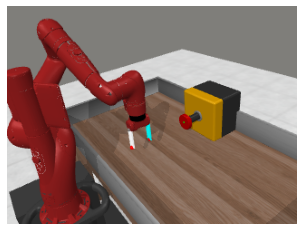}
\includegraphics[width=0.25\linewidth]{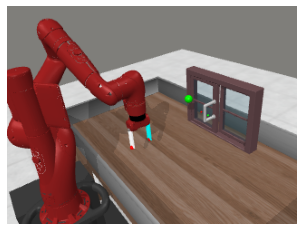}
\\
\subfloat[Door\_open]{\hspace{0.256\linewidth}}
\subfloat[Button\_press]{\hspace{0.256\linewidth}}
\subfloat[Window\_open]{\hspace{0.256\linewidth}}

\caption{Seven tasks from DMControl (a-c) and Metaworld (d-f).}
\label{fig:env examples}
\end{figure*}

\noindent \textbf{DMControl Tasks:}
\begin{enumerate} 
    \item \verb|Cheetah_run|: A planar biped is trained to control its body and run on the ground.
    \item \verb|Walker_run|: A planar walker is trained to control its body and walk on the ground.
    \item \verb|Quadruped_walk|: A four-legged ant is trained to control its body and limbs, and crawl slowly on the ground.
    \item \verb|Quadruped_run|: A four-legged ant is trained to control its body and limbs, and crawl fast on the ground.
\end{enumerate}

\noindent \textbf{Metaworld Tasks:}
\begin{enumerate} 
    \item \verb|Door_open|: An agent controls a robotic arm to open a door with a random position.
    \item \verb|Button_press|: An agent controls a robotic arm to press a button at a random position.
    \item \verb|Window_open|: An agent controls a robotic arm to open a window with a random position.
\end{enumerate}

\subsection{Implementation Details}
\label{app:implement_detail}

The APS, DIAYN and CIC algorithms used in \method are implemented based on the official repository of URLB \cite{laskin2021urlb}. We changed the DDPG backbone of URLB to SAC, as PEBBLE, SURF, and RUNE are all based on SAC. Also, for APS, to encourage exploration in the pretraining stage, we add a hyperparameter $\beta>1$ for the entropy term in intrinsic reward (Eq. \ref{eq:aps_intrinsic_reward}):
\begin{equation}
    r^\text{int}(s,a,s') = \phi(s)^Tz - \beta\log p(s'),
\end{equation}

For the implementation of SAC, PEBBLE, SURF, and RUNE, we refer to their corresponding official repositories and re-implement them within the URLB framework. For RIME, we use its official implementation.
SAC serves as a performance upper bound because it uses a ground-truth reward function, which is unavailable in PbRL settings for training. The detailed hyperparameters of SAC are shown in Table \ref{table:hyperparameters_sac}. 
PEBBLE's settings remain consistent with its original implementation, and the specifics are detailed in Table \ref{table:hyperparameters_pebble}. 
For SURF, RUNE, RIME and \method, most hyperparameters are the same as those of PEBBLE, and other hyperparameters are detailed in Table \ref{table:hyperparameters_surf}, \ref{table:hyperparameters_rune}, \ref{table:hyperparameters_rime} and \ref{table:hyperparameters_sepoa}, respectively. 
The total amount of feedback and feedback amount per session are detailed in Table \ref{table:hyperparameters_condition}.

For the component analysis of ablation studies in Section. \ref{subsec:ablation}, we remove skill-based unsupervised pretraining from \method (green and blue curves) by setting the pretraining step = 0 and randomly initializing the APS policy.
As for the skill-based query selection, we replace it with the disagreement scheme in PEBBLE and SURF (orange and blue curves).

Code of \method is available at \texttt{\url{https://github.com/MoonOutCloudBack/SEPOA\_PbRL}}.

\begin{table}[ht]
\centering
\begin{tabular}{ll}
\toprule
\textbf{Hyperparameter}~~~~~~~~~~~~~~~~~~~~ & \textbf{Value}~~~~~~~~~~~~~~~~~~~~ \\
\midrule
Number of layers            & $2$ (DMControl), \\
                            & $3$ (Metaworld) \\
Hidden units per layer      & $1024$ (DMControl), $256$ (Metaworld) \\
Activation function         & ReLU \\
Optimizer                   & Adam \\ 
Learning rate               & $0.0005$ (DMControl), $0.0001$ (Metaworld) \\
Initial temperature         & $0.2$ \\
Critic target update freq   & $2$ \\
Critic EMA $\tau$           & $0.01$ \\ 
Batch Size                  & $1024$ (DMControl), $512$ (Metaworld) \\
$(\beta_1,\beta_2)$         & $(0.9, 0.999)$ \\
Discount $\gamma$           & $0.99$ \\
\bottomrule
\end{tabular}
\caption{Hyperparameters of SAC.}
\label{table:hyperparameters_sac}
\end{table}

\begin{table}[ht]
\centering
\begin{tabular}{ll}
\toprule
\textbf{Hyperparameter} & \textbf{Value}~~~~~~~~~~~~~~~~~~~~~~~~~~~~~~~~~ \\
\midrule
Segment length                  & $50$ \\
Learning rate                   & $0.0005$ (DMControl), $0.0001$ (Metaworld) \\
Feedback frequency              & $20000$ (DMControl), $5000$ (Metaworld) \\
Num of reward ensembles         & $3$ \\
Reward model activator          & tanh \\
Unsupervised pretraining steps  & $9000$ \\
\bottomrule
\end{tabular}
\caption{Hyperparameters of PEBBLE.}
\label{table:hyperparameters_pebble}
\end{table}

\begin{table}[ht]
\centering
\begin{tabular}{ll}
\toprule
\textbf{Hyperparameter} & \textbf{Value}~~~~~~~~~~~~~~~~~ \\
\midrule
Unlabeled batch ratio $\mu$         & $4$ \\
Threshold $\tau$                    & $0.999$  \\
Loss weight $\lambda$               & 1 \\
Min/Max length of cropped segment   & $45/55$ \\
Segment length before cropping      & $60$ \\
\bottomrule
\end{tabular}
\caption{Hyperparameters of SURF.}
\label{table:hyperparameters_surf}
\end{table}

\begin{table}[ht]
\centering
\begin{tabular}{ll}
\toprule
\textbf{Hyperparameter}~~~~ & \textbf{Value}~~~~~~~~~~~~~~~~~~ \\
\midrule
Initial weight of intrinsic reward $\beta_0$    & $0.05$ \\
Decay rate $\rho$    & $0.001$  \\ 
\bottomrule
\end{tabular}
\caption{Hyperparameters of RUNE.}
\label{table:hyperparameters_rune}
\end{table}

\begin{table}[ht]
\centering
\begin{tabular}{ll}
\toprule
\textbf{Hyperparameter} & \textbf{Value} \\
\midrule
Coefficient $\alpha$ in the lower bound $\tau_\text{lower}$ & $0.5$ \\
Minimum weight $\beta_{\min}$ & 1 \\
Maximum weight $\beta_{\max}$ & 3 \\
Decay rate $k$ & 1/30 (DMControl), 1/300 (Metaworld) \\
Upper bound $\tau_\text{upper}$ & $3\ln(10)$ \\
$\delta$ for the intrinsic reward & $1 \times 10^{-8}$ \\
Steps of unsupervised pre-training & $2000$ (Cheetah), $9000$ (others) \\
\bottomrule
\end{tabular}
\caption{Hyperparameters of RIME.}
\label{table:hyperparameters_rime}
\end{table}

\begin{table}[ht]
\centering
\begin{tabular}{ll}
\toprule
\textbf{Hyperparameter}~~~~~~~~~~~~~~~~~~~~~~ & \textbf{Value}~~~~~~~~~~~~~~~~~~~~~~ \\
\midrule
Dim of task vector $z$                      & $10$ \\ 
APS pretraining steps                       & \\
~~~~-~~~~ \verb|Cheetah_run|                        & $5\times 10^5$ \\
~~~~-~~~~ \verb|Walker_run|                         & $5\times 10^5$ \\
~~~~-~~~~ \verb|Window_open|                        & $5\times 10^5$ \\
~~~~-~~~~ \verb|Button_press|                       & $5\times 10^5$ \\
~~~~-~~~~ \verb|Quadruped_walk|                     & $1\times 10^6$ \\
~~~~-~~~~ \verb|Quadruped_run|                      & $1\times 10^6$ \\
~~~~-~~~~ \verb|Door_open|                          & $1\times 10^6$ \\
$\beta$ of entropy term                     & 5 \\
Number of skills sampling $N_z$             & $50$ \\
Number of layers for $R_{\theta_r}$         & $3$ \\
Hidden layer dim for $R_{\theta_r}$         & $256$ \\
\bottomrule
\end{tabular}
\caption{Hyperparameters of \method.}
\label{table:hyperparameters_sepoa}
\end{table}

\begin{table}[ht]
\centering
\begin{tabular}{ll}
\toprule
\textbf{Environment}~~~~~~~~~~~~~~~~~~~~~~~~~~~~~ & \textbf{Value}~~~~~~~~~~~~~~~~~~~~~~  \\
\midrule
\verb|Cheetah_run|  & $1000/100$ \\ 
\verb|Walker_run|  & $1000/100$ \\ 
\verb|Quadruped_walk|  & $2000/200$ \\ 
\verb|Quadruped_run|   & $2000/200$ \\ 
\verb|Door_open|  & $30000/100$ \\ 
\verb|Button_press| & $30000/100$ \\ 
\verb|Window_open|  & $30000/100$  \\ 
\bottomrule
\end{tabular}
\caption{Feedback amount in each environment. The ``value" column refers to the feedback amount in total / per session.}
\label{table:hyperparameters_condition}
\end{table}

\clearpage

%% file: appendix/3_human_exp.tex
\section{Human Experiments}
\label{app:human_exp}

\paragraph{Preference collection.} 
We collect feedback from human labelers (the authors) familiar with the tasks. 
Specifically, the human labelers watch a video rendering each segment and select the one that better achieves the objective. 
Each trajectory segment is 1.5 seconds long (50 timesteps). 
For Figure \ref{fig:human_get_confused}, the labelers provide labels for 60 queries for each reward difference. 
For Figure \ref{fig:human_like_sepoa}, we run 3 random seeds for each method, with labelers providing 20 preference labels for each run.

\paragraph{Instruction given to human teachers.}

\begin{itemize}
    \item \verb|Cheetah_run|: The goal of the cheetah agent is to run as fast as possible. If the cheetah maintains a higher speed, prioritize that segment. If the speeds are similar, choose the segment where the cheetah has traveled further. Note that the exact speed threshold (\verb|_RUN_SPEED|) is unknown, but faster speeds are preferable.
    \item \verb|Walker_run|: The walker agent should first be standing upright. If both agents are upright or neither is, the tie breaker is the one moving faster to the right. The goal is to run, so prioritize the segment with the higher speed. If the speeds are close, choose the segment that covers more distance.
    \item \verb|Quadruped_walk|: The primary goal is for the agent to stand upright. If both agents are standing or neither is, prioritize the one with the higher speed, but ensure that the agent is moving at least at the walk speed threshold. If both agents are moving at the same speed, choose the segment where the agent’s center is higher off the ground.
    \item \verb|Quadruped_run|: Similar to \verb|Quadruped_walk|, but the agent must meet the running speed threshold. The priority is first to stand upright, then to move faster to the right. If both agents are moving at similar speeds, choose the one with the higher speed (above the running threshold). If they are tied in speed, prioritize the one with the higher stance.
\end{itemize}

%% file: appendix/4_skill_discovery.tex
\section{Details of skill discovery methods}
\label{app:skill_discovery}

In this section, we introduce the unsupervised skill discovery methods used in this paper in detail.

\subsection{APS}

\begin{wrapfigure}[13]{r}{0.5\textwidth}  
\centering
\vspace*{-0.3in}
\begin{minipage}{1\linewidth}
\centering
\begin{algorithm}[H]
\caption{\textsc{Unsupervised Pretrain} using APS}
\label{alg:aps_pretrain}
\begin{algorithmic}[1]
\FOR {each iteration}
    \STATE Randomly sample skill $z$
    \FOR {each environment step}
        \STATE Obtain $a_t\sim \pi(a|s_t,z)$ and $(s_t,a_t,s_{t+1})$
        \STATE Calculate intrinsic reward $r^\text{int}$ using Eq.~\ref{eq:aps_intrinsic_reward}
        \STATE Store transitions $\mathcal{B}\gets \mathcal{B}\cup \{(s_t,a_t,s_{t+1}, r^\text{int}_t, z)\}$
        \STATE Minimize $\mathcal{L} = -\phi(s_{t+1})^Tz$
        \STATE Minimize $\mathcal{L}_\text{critic}$ and $\mathcal{L}_\text{actor}$ with $r^\text{int}$ in Eq.~\ref{eq: sf q a} 
    \ENDFOR
\ENDFOR
\end{algorithmic}
\end{algorithm}
\end{minipage}
\end{wrapfigure}

APS \cite{liu2021aps} parameterizes latent space $z$ in the form of successor feature \cite{barreto2017successor,hansen2019fast,liu2021aps,laskin2021urlb}.
Specifically, APS assumes the reward function $r$ is the inner product of some state feature $\phi(s)$ and latent variable $z$ indicating the task \cite{barreto2017successor}.
The latent variable $z$ and state feature $\phi(s)$ are normalized to be the unit length \cite{hansen2019fast,liu2021aps}.

In the pretraining stage, APS learns a continuous skill space and a policy $\pi(\cdot|s,z)$ conditioned on the skill by maximizing the intrinsic reward in Eq. \ref{eq:aps_intrinsic_reward}. 
The intrinsic reward includes the inner product between state features $\phi(s)$ and skills $z$, and state entropy $H(\phi(s'))$ \cite{liu2021aps}. 
Besides, APS parameterizes $q(s|z)$ as the Von Mises-Fisher distribution with a scale parameter of $1$ \cite{hansen2019fast}.
The intrinsic reward function is
\begin{equation}
\label{eq:aps_intrinsic_reward}
    r^\text{int}(s,a,s')=\phi(s)^Tz-\log p(s'),
\end{equation}
where $-\log p(s')$ is the entropy term.
APS estimates $-\log p(s')$ using particle-based entropy estimation \cite{singh2003nearest,liu2021behavior}:
\begin{equation}
    -\log p(s')=\log\left(1+\frac1k\sum_{k}\|h-h^k\|\right),
\end{equation}
where $h=\phi(s')$ is the feature of the successor state $s'$, $h^k$ denotes the $k$-th nearest neighbors of $h$ in the replay buffer.
The specific pretraining process using APS is shown in Algorithm \ref{alg:aps_pretrain}.

\paragraph{Successor Feature.}

Successor feature \cite{barreto2017successor} assumes the reward function $r$ is the inner product of some state feature $\phi(s)$ and latent variable $z$ indicating the task:
\begin{equation}
    r(s,a)=\phi(s')^Tz.
    \label{eq: sf r}
\end{equation}
With successor feature, the state-action value function is decomposed as follows:
\begin{align}
Q^{\pi}(s,a)& =\mathbb{E}_{s_0=s, a_0=a} \left[\sum_{i=0}^\infty\gamma^{i}\phi(s_{i+1},a_{i+1},s_{i+1}')\right]^Tz \nonumber \\
&\equiv\Psi^\pi(s,a)^Tz,
\label{eq: sf q}
\end{align}
where $\Psi^\pi(s,a) \equiv \mathbb{E}_{s_0=s, a_0=a} \sum_{i=0}^\infty\gamma^{i} \phi(s_{i+1},a_{i+1},s_{i+1}')$ is the successor feature of $\pi$. 
The critic and the actor could be updated by minimizing the following critic and actor loss:
\begin{align}
\mathcal{L}_{\rm critic}&=\|\Psi(s_t,a_t)^Tz - r_t - \gamma\Psi(s_{t+1},\pi(s_{t+1},z))^T z\|_2 \nonumber \\
\mathcal{L}_{\rm actor}&=-\Psi(s_t,\pi(s_t,z))^Tz
\label{eq: sf q a}
\end{align}
%

\subsection{DIAYN} 
DIAYN \cite{eysenbach2019diversity} considers a discrete skill space, and maximizes
\begin{equation}
    I(S;Z) + \mathcal{H}[A \mid S] - I(A;Z \mid S)=I(S;Z) + \mathcal{H}[A \mid S] - I(A;Z \mid S),
\end{equation}
where $\mathcal{H}$ is entropy, $I$ is the mutual information, and $S$, $A$, $Z$ are random variables for states, actions, and skills respectively.
For the first and the last terms, i.e. $I(S;Z) - I(A;Z \mid S)$, the induced intrinsic reward is
\begin{equation}
    r_z(s, a) = \log q_\phi(z \mid s) - \log p(z),
\end{equation}
where $p(z)$ is the density function of a uniform distribution.
And the policy is updated with SAC to maximize $\mathcal{H}[A \mid S]$.

\subsection{CIC} 
CIC \cite{laskin2022cic} considers a continuous skill space and optimizes the forward form of the mutual information to explicitly optimize the state coverage of the learned skills.
It parameterizes the discriminator with a contrastive density estimator:
\begin{equation}
    \log q(\tau|z) = f(\tau, z) - \log \frac{1}{N} \sum_{j=1}^{N} \exp(f(\tau_j, z)),
\end{equation}
where $f(\tau, z)=g_{\psi_1}(\tau)^\top g_{\psi_2}(z) / (\| g_{\psi_1}(\tau) \| \| g_{\psi_2}(z) \| T)$, $\tau=(s,s')$ is a transition and $T$ is a temperature coefficient.
CIC further uses contrastive learning methods to learn the discriminator, and optimize the entropy term in the forward form of the mutual information using particle-based entropy estimation as \cite{singh2003nearest,liu2021behavior,liu2021aps}.

%% file: main.bbl
\begin{thebibliography}{}

\bibitem[\protect\citeauthoryear{Barreto \bgroup \em et al.\egroup }{2017}]{barreto2017successor}
Andr{\'e} Barreto, Will Dabney, R{\'e}mi Munos, Jonathan~J Hunt, Tom Schaul, Hado~P van Hasselt, and David Silver.
\newblock Successor features for transfer in reinforcement learning.
\newblock {\em Advances in neural information processing systems}, 30, 2017.

\bibitem[\protect\citeauthoryear{Bellemare \bgroup \em et al.\egroup }{2016}]{bellemare2016unifying}
Marc Bellemare, Sriram Srinivasan, Georg Ostrovski, Tom Schaul, David Saxton, and Remi Munos.
\newblock Unifying count-based exploration and intrinsic motivation.
\newblock {\em Advances in neural information processing systems}, 29, 2016.

\bibitem[\protect\citeauthoryear{Bellemare \bgroup \em et al.\egroup }{2020}]{bellemare2020autonomous}
Marc~G Bellemare, Salvatore Candido, Pablo~Samuel Castro, Jun Gong, Marlos~C Machado, Subhodeep Moitra, Sameera~S Ponda, and Ziyu Wang.
\newblock Autonomous navigation of stratospheric balloons using reinforcement learning.
\newblock {\em Nature}, 588(7836):77--82, 2020.

\bibitem[\protect\citeauthoryear{Biyik \bgroup \em et al.\egroup }{2020}]{biyik2020active}
Erdem Biyik, Nicolas Huynh, Mykel Kochenderfer, and Dorsa Sadigh.
\newblock Active preference-based gaussian process regression for reward learning.
\newblock In {\em Robotics: Science and Systems}, 2020.

\bibitem[\protect\citeauthoryear{Bradley and Terry}{1952}]{bradley-terry}
Ralph~Allan Bradley and Milton~E Terry.
\newblock Rank analysis of incomplete block designs: I. the method of paired comparisons.
\newblock {\em Biometrika}, 39(3/4):324--345, 1952.

\bibitem[\protect\citeauthoryear{Burda \bgroup \em et al.\egroup }{2019}]{burda2019exploration}
Yuri Burda, Harrison Edwards, Amos Storkey, and Oleg Klimov.
\newblock Exploration by random network distillation.
\newblock In {\em Seventh International Conference on Learning Representations}, pages 1--17, 2019.

\bibitem[\protect\citeauthoryear{Chen \bgroup \em et al.\egroup }{2022}]{chen2022towards}
Yuanpei Chen, Tianhao Wu, Shengjie Wang, Xidong Feng, Jiechuan Jiang, Zongqing Lu, Stephen McAleer, Hao Dong, Song-Chun Zhu, and Yaodong Yang.
\newblock Towards human-level bimanual dexterous manipulation with reinforcement learning.
\newblock {\em Advances in Neural Information Processing Systems}, 35:5150--5163, 2022.

\bibitem[\protect\citeauthoryear{Cheng \bgroup \em et al.\egroup }{2024}]{cheng2024rime}
Jie Cheng, Gang Xiong, Xingyuan Dai, Qinghai Miao, Yisheng Lv, and Fei-Yue Wang.
\newblock Rime: Robust preference-based reinforcement learning with noisy preferences.
\newblock {\em arXiv preprint arXiv:2402.17257}, 2024.

\bibitem[\protect\citeauthoryear{Christiano \bgroup \em et al.\egroup }{2017}]{christiano2017deep}
Paul~F Christiano, Jan Leike, Tom Brown, Miljan Martic, Shane Legg, and Dario Amodei.
\newblock Deep reinforcement learning from human preferences.
\newblock {\em Advances in neural information processing systems}, 30, 2017.

\bibitem[\protect\citeauthoryear{Eysenbach \bgroup \em et al.\egroup }{2019}]{eysenbach2019diversity}
Benjamin Eysenbach, Julian Ibarz, Abhishek Gupta, and Sergey Levine.
\newblock Diversity is all you need: Learning skills without a reward function.
\newblock In {\em 7th International Conference on Learning Representations, ICLR 2019}, 2019.

\bibitem[\protect\citeauthoryear{Fulton and Platzer}{2018}]{fulton2018safe}
Nathan Fulton and Andr{\'e} Platzer.
\newblock Safe reinforcement learning via formal methods: Toward safe control through proof and learning.
\newblock In {\em Proceedings of the AAAI Conference on Artificial Intelligence}, volume~32, 2018.

\bibitem[\protect\citeauthoryear{Hansen \bgroup \em et al.\egroup }{2019}]{hansen2019fast}
Steven Hansen, Will Dabney, Andre Barreto, Tom Van~de Wiele, David Warde-Farley, and Volodymyr Mnih.
\newblock Fast task inference with variational intrinsic successor features.
\newblock {\em arXiv preprint arXiv:1906.05030}, 2019.

\bibitem[\protect\citeauthoryear{Hazan \bgroup \em et al.\egroup }{2019}]{hazan2019provably}
Elad Hazan, Sham Kakade, Karan Singh, and Abby Van~Soest.
\newblock Provably efficient maximum entropy exploration.
\newblock In {\em International Conference on Machine Learning}, pages 2681--2691. PMLR, 2019.

\bibitem[\protect\citeauthoryear{Hejna~III and Sadigh}{2023}]{hejna2023few}
Donald~Joseph Hejna~III and Dorsa Sadigh.
\newblock Few-shot preference learning for human-in-the-loop rl.
\newblock In {\em Conference on Robot Learning}, pages 2014--2025. PMLR, 2023.

\bibitem[\protect\citeauthoryear{Huber}{2020}]{huber2020bayesian}
Marco~F Huber.
\newblock Bayesian perceptron: Towards fully bayesian neural networks.
\newblock In {\em 2020 59th IEEE Conference on Decision and Control (CDC)}, pages 3179--3186. IEEE, 2020.

\bibitem[\protect\citeauthoryear{Ibarz \bgroup \em et al.\egroup }{2018}]{ibarz2018reward}
Borja Ibarz, Jan Leike, Tobias Pohlen, Geoffrey Irving, Shane Legg, and Dario Amodei.
\newblock Reward learning from human preferences and demonstrations in atari.
\newblock {\em Advances in neural information processing systems}, 31, 2018.

\bibitem[\protect\citeauthoryear{Kim \bgroup \em et al.\egroup }{2023}]{kim2023preference}
Changyeon Kim, Jongjin Park, Jinwoo Shin, Honglak Lee, Pieter Abbeel, and Kimin Lee.
\newblock Preference transformer: Modeling human preferences using transformers for rl.
\newblock {\em arXiv preprint arXiv:2303.00957}, 2023.

\bibitem[\protect\citeauthoryear{Kristiadi \bgroup \em et al.\egroup }{2020}]{kristiadi2020being}
Agustinus Kristiadi, Matthias Hein, and Philipp Hennig.
\newblock Being bayesian, even just a bit, fixes overconfidence in relu networks.
\newblock In {\em International conference on machine learning}, pages 5436--5446. PMLR, 2020.

\bibitem[\protect\citeauthoryear{Laskin \bgroup \em et al.\egroup }{2021}]{laskin2021urlb}
Michael Laskin, Denis Yarats, Hao Liu, Kimin Lee, Albert Zhan, Kevin Lu, Catherine Cang, Lerrel Pinto, and Pieter Abbeel.
\newblock Urlb: Unsupervised reinforcement learning benchmark.
\newblock {\em arXiv preprint arXiv:2110.15191}, 2021.

\bibitem[\protect\citeauthoryear{Laskin \bgroup \em et al.\egroup }{2022}]{laskin2022cic}
Michael Laskin, Hao Liu, Xue~Bin Peng, Denis Yarats, Aravind Rajeswaran, and Pieter Abbeel.
\newblock Unsupervised reinforcement learning with contrastive intrinsic control.
\newblock {\em Advances in Neural Information Processing Systems}, 35:34478--34491, 2022.

\bibitem[\protect\citeauthoryear{Lee \bgroup \em et al.\egroup }{2021a}]{lee2021b}
Kimin Lee, Laura Smith, Anca Dragan, and Pieter Abbeel.
\newblock B-pref: Benchmarking preference-based reinforcement learning.
\newblock {\em arXiv preprint arXiv:2111.03026}, 2021.

\bibitem[\protect\citeauthoryear{Lee \bgroup \em et al.\egroup }{2021b}]{lee2021pebble}
Kimin Lee, Laura~M Smith, and Pieter Abbeel.
\newblock Pebble: Feedback-efficient interactive reinforcement learning via relabeling experience and unsupervised pre-training.
\newblock In {\em International Conference on Machine Learning}, pages 6152--6163. PMLR, 2021.

\bibitem[\protect\citeauthoryear{Liang \bgroup \em et al.\egroup }{2022}]{liang2022reward}
Xinran Liang, Katherine Shu, Kimin Lee, and Pieter Abbeel.
\newblock Reward uncertainty for exploration in preference-based reinforcement learning.
\newblock {\em arXiv preprint arXiv:2205.12401}, 2022.

\bibitem[\protect\citeauthoryear{Liu and Abbeel}{2021a}]{liu2021aps}
Hao Liu and Pieter Abbeel.
\newblock Aps: Active pretraining with successor features.
\newblock In {\em International Conference on Machine Learning}, pages 6736--6747. PMLR, 2021.

\bibitem[\protect\citeauthoryear{Liu and Abbeel}{2021b}]{liu2021behavior}
Hao Liu and Pieter Abbeel.
\newblock Behavior from the void: Unsupervised active pre-training.
\newblock {\em Advances in Neural Information Processing Systems}, 34:18459--18473, 2021.

\bibitem[\protect\citeauthoryear{Liu \bgroup \em et al.\egroup }{2023}]{liu2023efficient}
Yi~Liu, Gaurav Datta, Ellen Novoseller, and Daniel~S Brown.
\newblock Efficient preference-based reinforcement learning using learned dynamics models.
\newblock In {\em 2023 IEEE International Conference on Robotics and Automation (ICRA)}, pages 2921--2928. IEEE, 2023.

\bibitem[\protect\citeauthoryear{Luan \bgroup \em et al.\egroup }{2025}]{luan2025efficient}
Yao Luan, Qing-Shan Jia, Yi~Xing, Zhiyu Li, and Tengfei Wang.
\newblock An efficient real-time railway container yard management method based on partial decoupling.
\newblock {\em IEEE Transactions on Automation Science and Engineering}, 22:14183--14200, 2025.

\bibitem[\protect\citeauthoryear{Mnih \bgroup \em et al.\egroup }{2013}]{mnih2013playing}
Volodymyr Mnih, Koray Kavukcuoglu, David Silver, Alex Graves, Ioannis Antonoglou, Daan Wierstra, and Martin Riedmiller.
\newblock Playing atari with deep reinforcement learning.
\newblock {\em arXiv preprint arXiv:1312.5602}, 2013.

\bibitem[\protect\citeauthoryear{Mu \bgroup \em et al.\egroup }{2024}]{mu2024large}
Ni~Mu, Xiao Hu, Qing-Shan Jia, Xu~Zhu, and Xiao He.
\newblock Large-scale data center cooling control via sample-efficient reinforcement learning.
\newblock In {\em 2024 IEEE 20th International Conference on Automation Science and Engineering (CASE)}, pages 2780--2785. IEEE, 2024.

\bibitem[\protect\citeauthoryear{Ostrovski \bgroup \em et al.\egroup }{2017}]{ostrovski2017count}
Georg Ostrovski, Marc~G Bellemare, A{\"a}ron Oord, and R{\'e}mi Munos.
\newblock Count-based exploration with neural density models.
\newblock In {\em International conference on machine learning}, pages 2721--2730. PMLR, 2017.

\bibitem[\protect\citeauthoryear{Park \bgroup \em et al.\egroup }{2022a}]{park2022surf}
Jongjin Park, Younggyo Seo, Jinwoo Shin, Honglak Lee, Pieter Abbeel, and Kimin Lee.
\newblock Surf: Semi-supervised reward learning with data augmentation for feedback-efficient preference-based reinforcement learning.
\newblock {\em arXiv preprint arXiv:2203.10050}, 2022.

\bibitem[\protect\citeauthoryear{Park \bgroup \em et al.\egroup }{2022b}]{park2022lipschitz}
Seohong Park, Jongwook Choi, Jaekyeom Kim, Honglak Lee, and Gunhee Kim.
\newblock Lipschitz-constrained unsupervised skill discovery.
\newblock In {\em International Conference on Learning Representations}, 2022.

\bibitem[\protect\citeauthoryear{Park \bgroup \em et al.\egroup }{2023}]{park2023metra}
Seohong Park, Oleh Rybkin, and Sergey Levine.
\newblock Metra: Scalable unsupervised rl with metric-aware abstraction.
\newblock {\em arXiv preprint arXiv:2310.08887}, 2023.

\bibitem[\protect\citeauthoryear{Pathak \bgroup \em et al.\egroup }{2017}]{pathak2017curiosity}
Deepak Pathak, Pulkit Agrawal, Alexei~A Efros, and Trevor Darrell.
\newblock Curiosity-driven exploration by self-supervised prediction.
\newblock In {\em International conference on machine learning}, pages 2778--2787. PMLR, 2017.

\bibitem[\protect\citeauthoryear{Pathak \bgroup \em et al.\egroup }{2019}]{pathak2019self}
Deepak Pathak, Dhiraj Gandhi, and Abhinav Gupta.
\newblock Self-supervised exploration via disagreement.
\newblock In {\em International conference on machine learning}, pages 5062--5071. PMLR, 2019.

\bibitem[\protect\citeauthoryear{Sharma \bgroup \em et al.\egroup }{2020}]{sharma2020dynamics}
Archit Sharma, Shixiang Gu, Sergey Levine, Vikash Kumar, and Karol Hausman.
\newblock Dynamics-aware unsupervised discovery of skills.
\newblock In {\em International Conference on Learning Representations}, 2020.

\bibitem[\protect\citeauthoryear{Shin \bgroup \em et al.\egroup }{2023}]{shin2023benchmarks}
Daniel Shin, Anca~D Dragan, and Daniel~S Brown.
\newblock Benchmarks and algorithms for offline preference-based reward learning.
\newblock {\em arXiv preprint arXiv:2301.01392}, 2023.

\bibitem[\protect\citeauthoryear{Silver \bgroup \em et al.\egroup }{2016}]{silver2016mastering}
David Silver, Aja Huang, Chris~J Maddison, Arthur Guez, Laurent Sifre, George Van Den~Driessche, Julian Schrittwieser, Ioannis Antonoglou, Veda Panneershelvam, Marc Lanctot, et~al.
\newblock Mastering the game of go with deep neural networks and tree search.
\newblock {\em nature}, 529(7587):484--489, 2016.

\bibitem[\protect\citeauthoryear{Singh \bgroup \em et al.\egroup }{2003}]{singh2003nearest}
Harshinder Singh, Neeraj Misra, Vladimir Hnizdo, Adam Fedorowicz, and Eugene Demchuk.
\newblock Nearest neighbor estimates of entropy.
\newblock {\em American journal of mathematical and management sciences}, 23(3-4):301--321, 2003.

\bibitem[\protect\citeauthoryear{Tassa \bgroup \em et al.\egroup }{2018}]{tassa2018dmcontrol}
Yuval Tassa, Yotam Doron, Alistair Muldal, Tom Erez, Yazhe Li, Diego de~Las Casas, David Budden, Abbas Abdolmaleki, Josh Merel, Andrew Lefrancq, et~al.
\newblock Deepmind control suite.
\newblock {\em arXiv preprint arXiv:1801.00690}, 2018.

\bibitem[\protect\citeauthoryear{Xie \bgroup \em et al.\egroup }{2022}]{xie2022pretraining}
Zhihui Xie, Zichuan Lin, Junyou Li, Shuai Li, and Deheng Ye.
\newblock Pretraining in deep reinforcement learning: A survey.
\newblock {\em arXiv preprint arXiv:2211.03959}, 2022.

\bibitem[\protect\citeauthoryear{Yu \bgroup \em et al.\egroup }{2020}]{yu2020metaworld}
Tianhe Yu, Deirdre Quillen, Zhanpeng He, Ryan Julian, Karol Hausman, Chelsea Finn, and Sergey Levine.
\newblock Meta-world: A benchmark and evaluation for multi-task and meta reinforcement learning.
\newblock In {\em Conference on robot learning}, pages 1094--1100. PMLR, 2020.

\end{thebibliography}
